\pdfoutput=1
\documentclass[11pt]{article}
\usepackage[letterpaper,margin=1in]{geometry}

\usepackage{lipsum}
\usepackage{caption}

\usepackage[utf8]{inputenc} %
\usepackage[T1]{fontenc}    %

\usepackage{url}            %
\usepackage{booktabs}       %
\usepackage{nicefrac}       %
\usepackage{microtype}      %
\usepackage{enumitem}
\usepackage{dsfont}
\usepackage{graphicx}
\usepackage{multicol}
\usepackage{xcolor}

\usepackage{amsthm,amsfonts,amsmath,amssymb,epsfig,color,float,graphicx,verbatim, enumitem, subfigure}

\usepackage{algorithm,algorithmicx}
\usepackage[noend]{algpseudocode}

\usepackage{mathtools}
\usepackage{bbm}

\usepackage{thm-restate}
\usepackage{turnstile}

\usepackage[colorlinks,citecolor=blue,linkcolor=magenta,bookmarks=true,hypertexnames=false]{hyperref}
\usepackage[nameinlink]{cleveref}

\crefname{equation}{equation}{equations}
\crefname{lemma}{lemma}{lemmata}
\crefname{claim}{claim}{claims}
\crefname{theorem}{theorem}{theorems}
\crefname{proposition}{proposition}{propositions}
\crefname{corollary}{corollary}{corollaries}
\crefname{claim}{claim}{claims}
\crefname{remark}{remark}{remarks}
\crefname{definition}{definition}{definitions}
\crefname{fact}{fact}{facts}
\crefname{question}{question}{questions}
\crefname{condition}{condition}{conditions}
\crefname{algorithm}{algorithm}{algorithms}
\crefname{assumption}{assumption}{assumptions}
\crefname{problem}{problem}{problems}

\newtheorem{theorem}{Theorem}[section]
\newtheorem{lemma}[theorem]{Lemma}
\newtheorem{proposition}[theorem]{Proposition}
\newtheorem{corollary}[theorem]{Corollary}
\newtheorem{claim}[theorem]{Claim}

\newtheorem{definition}[theorem]{Definition}

\newtheorem{fact}[theorem]{Fact}

\theoremstyle{definition}
\newtheorem{assumption}[theorem]{Assumption}

\newtheorem{remark}[theorem]{Remark}

\usepackage{comment}

\newcommand{\cons}{\text{\cons}}

\newcommand{\eps}{\epsilon}

\newcommand{\poly}{\mathrm{poly}}

\newcommand{\cov}{\mathbf{Cov}}

\newcommand{\Var}{\mathbf{Var}}

\def\D{\mathcal D}

\def\R{\mathbb R}

\def\op{\text{op}}
\def\N{\mathbb N}

\def\Z{\mathbb Z}

\newcommand{\cA}{\mathcal{A}}

\newcommand{\cE}{\mathcal{E}}

\newcommand{\cN}{\mathcal{N}}

\newcommand{\iid}{{i.i.d.}\ }

\newcommand{\Paren}[1]{\left(#1\right)}

\newcommand{\Brac}[1]{\left[#1\right]}

\newcommand{\abs}[1]{\lvert#1\rvert}

\newcommand{\Set}[1]{\left\{#1\right\}}

\newcommand{\norm}[1]{\lVert#1\rVert}

\newcommand{\normt}[1]{\norm{#1}_2}

\newcommand{\iprod}[1]{\langle#1\rangle}

\newcommand{\hide}[1]{}

\DeclareMathOperator*{\E}{\mathbf{E}}

\newcommand{\littlesum}{\mathop{\textstyle \sum}}

\newcommand{\lp}{\left}
\newcommand{\rp}{\right}

\newcommand{\normal}{\mathcal{N}}

\let\vec\mathbf
\DeclarePairedDelimiter\ceil{\lceil}{\rceil}

\def\colorful{0}

\ifnum\colorful=1

\newcommand{\inote}[1]{\footnote{{\bf [Ilias: {#1}\bf ]}}}
\newcommand{\anote}[1]{\footnote{{\bf [Ankit: {#1}\bf ]}}}
\newcommand{\tnote}[1]{\footnote{{\bf [Thanasis: {#1}\bf ] }}}
\newcommand{\snote}[1]{\footnote{{\bf [Sushrut: {#1}\bf ] }}}

\newcommand{\lnote}[1]{\footnote{{\bf [Sihan: {#1}\bf ] }}}

\else

\newcommand{\inote}[1]{}
\newcommand{\anote}[1]{}
\newcommand{\tnote}[1]{}
\newcommand{\snote}[1]{}
\newcommand{\lnote}[1]{}

\fi
\newcommand{\IE}{\text{IE}}

\allowdisplaybreaks

\newcommand{\momb}{Q}

\newcommand{\marginal}{\mathcal G}

\newcommand{\constrain}{\{ \normt{v}^2 = 1 \}}

\newcommand{\batchsc}{ (4kd)^{8k} \momb^{-1} + 1}

\newcommand{\Mbound}{ O((2k)^{2k}/n^{k}) \; Q \; ( \sigma^{2k} + 2 \; \norm{\beta^*}_2^{2k} ) }

\newcommand{\betagap}{R + k \alpha^{-1/k} \sigma Q^{1/k} / \sqrt{n}}

\newcommand{\shortbetagap}{R + 
\frac{k \alpha^{-1/k} \sigma Q^{1/k}}{\sqrt{n}}}

\newcommand{\prunebs}{k \; Q^{2/k} \; \alpha^{-2/k}}

\newcommand{\sosbetagap}{  \lp({k}{ \sqrt{n} }\rp) \; Q^{1/(2k)} \;  (R + \sigma) \; \alpha^{-3/k}}

\title{Batch List-Decodable Linear Regression via Higher Moments}

\author{
Ilias Diakonikolas\thanks{Supported by NSF Medium Award CCF-2107079 and an H.I. Romnes Faculty Fellowship.}\\
University of Wisconsin-Madison\\
{\tt ilias@cs.wisc.edu}
\and
Daniel M. Kane\thanks{Supported by NSF Medium Award CCF-2107547 and NSF Award CCF-1553288 (CAREER).}\\
UC San Diego\\
{\tt dakane@ucsd.edu}
\and
Sushrut Karmalkar\thanks{Supported by NSF under Grant \#2127309 to the Computing Research Association for the CIFellows 2021 Project.}\\
University of Wisconsin-Madison\\
{\tt s.sushrut@gmail.com}
\and
Sihan Liu\\
UC San Diego\\
{\tt sil046@ucsd.edu}
\and
Thanasis Pittas\thanks{Supported by NSF Medium Award CCF-2107079 and NSF Award DMS-2023239 (TRIPODS).}\\
University of Wisconsin-Madison\\
{\tt pittas@wisc.edu}
}

\begin{document}
	
\maketitle

\begin{abstract}%
We study the task of list-decodable linear regression using batches. %
A batch is called clean if the points it contains are i.i.d.\ 
samples from an unknown linear regression distribution. 
For a parameter $\alpha \in (0, 1/2)$, an unknown $\alpha$-fraction 
of the batches are clean 
and no assumptions are made on the remaining batches. 
The goal is to output a small list of vectors at least one of which 
is close to the true regressor vector in $\ell_2$-norm. \cite{DasJKS22} gave an efficient algorithm for this task, under natural distributional assumptions, 
with the following guarantee. 
Under the assumption that the batch size $n$ 
satisfies $n \geq \tilde{\Omega}(\alpha^{-1})$ 
and the number of batches 
is $m = \poly(d, n, 1/\alpha)$, 
their algorithm runs in polynomial time and 
outputs a list of $O(1/\alpha^2)$ vectors at least one of which 
is $\tilde{O}(\alpha^{-1/2}/\sqrt{n})$ close to the target regressor. Here we design a new polynomial-time algorithm 
for this task with significantly stronger guarantees 
under the assumption that the low-degree moments 
of the covariates distribution are 
Sum-of-Squares (SoS) certifiably bounded. 
Specifically, for any constant $\delta>0$, as long as the batch size is 
$n \geq \Omega_{\delta}(\alpha^{-\delta})$
and the degree-$\Theta(1/\delta)$ moments of the covariates are SoS certifiably bounded, 
our algorithm uses $m = \poly((dn)^{1/\delta}, 1/\alpha)$ batches,
runs in polynomial-time, 
and outputs an $O(1/\alpha)$-sized list of vectors one of which is 
$O(\alpha^{-\delta/2}/\sqrt{n})$ close to the target. That is, %
our algorithm achieves substantially 
smaller minimum batch size 
and final error,  
while achieving the optimal list size. 
Our approach leverages higher-order moment information 
by carefully combining the SoS paradigm 
interleaved with an iterative method and a novel list  
pruning procedure for this setting. 
In the process, we give an SoS proof 
of the Marcinkiewicz-Zygmund inequality 
that may be of broader applicability. 
\end{abstract}

\setcounter{page}{0}

\thispagestyle{empty}

\newpage

\section{Introduction}
In several modern applications of data analysis, 
including
federated learning \cite{wang2021field}, 
sensor networks \cite{WaxZ89}, 
and crowdsourcing \cite{SteVC16},
it is typically infeasible to collect 
large datasets from a single source. 
Instead, samples are collected in batches from multiple sources. 
Unfortunately, it is often hard to find sources that provide many samples, i.e., that have large-size batches.
A standard example is a movie recommendation system
using rates collected from users.
Here, an individual user is often unlikely to provide rating scores for a large number of movies, frequently resulting in data batches with relatively small sizes.
Even less favorable, in such crowdsourcing settings, 
it is also the case that a \emph{majority} of the participants 
might be unreliable~\cite{SteKL17,SteVC16,CSV17}. 
Such practical scenarios serve as motivation for this work.

Formally, we study 
the task of linear regression
under the assumption that we are given access to the model through small batches of samples collected from different sources. 
Importantly, as motivated by our running example, we consider the setting where \emph{most} batches might not be collected from reliable sources.
Our formal setup is encapsulated in the following definition. 

\begin{definition}[List-Decodable Linear Regression using Batches]\label{def:glr} 
Let $D_{\beta^{*}}$ be the distribution on pairs $(X,y) \in \R^{d+1}$ such that  $y = {\beta^*}^\top X + \xi$, for $\xi \sim \cN(0,\sigma^2)$ and $X \sim \marginal$ that are drawn independently from each other.
Suppose we are given $m$ batches of size $n$ each, where for each batch, with probability $\alpha$ the batch consists entirely of \iid samples from $D_{\beta^{*}}$ and with probability $1-\alpha$  it is
drawn from some arbitrary distribution.
The goal is to output a list $L$ of vectors in $\R^d$ with $|L| \leq O(1/\alpha)$ and the guarantee that there is a $\widehat \beta \in L$ such that $\norm{\widehat \beta - \beta^*}_2$ is small. 
\end{definition}

For the vanilla setting of linear regression, with batch 
size $n=1$ and no outliers, the classical least-squares 
estimator is essentially optimal. Unfortunately, even a 
single outlier is enough to force the least-squares 
estimator to deviate arbitrarily. 
To address this discrepancy,~\cite{Hub64,Rousseeuw:1987} 
proposed classical robust estimators that could handle a 
constant fraction of outliers. However, these estimators are {computationally intractable (i.e., have runtime exponential in the dimension)}. 
Starting with the works of \cite{DKKLMS16, LaiRV16}, there have been a flurry of results
designing computationally efficient estimators in high dimensions 
which are robust to a small constant fraction 
of arbitrary outliers in the data.
See \cite{diakonikolas2023algorithmic} for an overview of this field.

The regime where a {\em majority} of the data might be outliers, known as list-decodable setting, was initially examined for mean estimation, where~\cite{CSV17} demonstrated the first polynomial-time algorithm under natural distributional assumptions.
Their algorithm computes a small list of hypotheses
with the guarantee that one element in the list is close to the target. Generating a list of candidates, as opposed to a single solution, is information-theoretically necessary in this regime (intuitively because the outliers can mimic legitimate data points).
The size of the list typically scales polynomially with the inverse of the inlier fraction, $\alpha$. 
The problem of list-decodable linear regression {(with batches of size $n=1$)}
was first studied in~\cite{KarKK19,RagYau19-list}.
Unfortunately, {the algorithms obtained in both of these works had 
sample and computational complexities scaling {\em exponentially } in $1/\alpha$, 
specifically of the form $d^{\poly(1/\alpha)}$ for $d$ dimensions.}
Interestingly, it was subsequently shown~\cite{DiaKPPS21} that such a dependence 
may be inherent {(for Statistical Query algorithms and low-degree polynomial tests---two powerful, yet restricted, models of computation)}.

{Motivated by this hardness result, \cite{DasJKS22} proposed the batch version of the list-decodable linear regression problem (\Cref{def:glr}). The hope was that by introducing (sufficiently large) batches, the exponential complexity dependence on $1/\alpha$ can be eliminated. Before we summarize their results, some comments are in order regarding 
\Cref{def:glr}.} First, in the extreme case where the batch size is $n=1$,  we 
recover the standard list-decodable setting. Second, in the other extreme 
where $n = \Omega(d)$, the problem becomes straightforward, 
since one batch contains sufficient information to recover the target regression 
vector. 
As discussed in our running example, 
the batch size, which corresponds to the number of samples collected from a single source, is rarely large enough in real-world applications with high-dimensional data. 
This leaves the regime of $1<n \ll d$ as the most 
meaningful. \cite{DasJKS22} showed that using $m = \poly(d, n, 1/\alpha)$ 
batches of size $n \geq \tilde \Omega(1/\alpha)$, it is possible to 
efficiently recover a list of size $O(1/\alpha^{2})$ containing an element 
$\hat{\beta}$ with $\| \hat{\beta} - \beta \|_2 = O(\sigma/ \sqrt{n\alpha})$. 
Their algorithm runs in fully polynomial time, thus escaping 
the exponential dependence on $1/\alpha$.

The linear dependence on $1/\alpha$ in the minimum batch size $n$ is 
inherent in the approach of~\cite{DasJKS22}. Motivated by the practical 
applications of the batch setting, here we ask whether efficient algorithms are 
possible that succeed with significantly smaller batch size and/or with better error guarantees:
\begin{center}
    \emph{Is there a {computationally efficient} algorithm for list-decodable linear regression in the batch setting with significantly improved batch size and/or error guarantees?}
\end{center}
{Here we answer this question in the affirmative. In particular, 
we provide an 
algorithm that for any constant $\delta>0$, it runs in polynomial time 
and succeeds with minimum batch size $n=\Theta_\delta(\alpha^{-\delta})$ 
achieving error $O_\delta(\sigma \alpha^{-\delta/2}/\sqrt{n})$. 
As a note regarding notation, we will switch from using the parameter $\delta>0$ to using $k =  \lceil 1/\delta \rceil $ throughout the paper.
}

\subsection{Our Results}
\label{subsec:our_results}
\noindent Throughout our work, we assume that the clean covariate distribution satisfies the following conditions.
\begin{assumption}
\label{ass:dist}
Let $X$ be the clean covariates distribution from \Cref{def:glr}.
We assume that
\begin{enumerate}[leftmargin=*]
\item $X$ is $L_4$-$L_2$ hypercontractive, i.e., for any $u \in \R^d$, 
it holds $\E \lp[ (u^\top X)^4 \rp] \leq O(1) \; \E\lp[ (u^\top X)^2 \rp]^2$.

\item $X$ has identity second moment, i.e., $\E[ X X^\top ] = \vec I$.

\item There exists $\momb \geq 1$ \footnote{Since $X$ has 
identity covariance, this implies that the bound $Q$ has to 
be at least $1$. } and an integer $\Delta$ such that for all 
integer $t \in [\Delta]$ the degree-$2t$ moments of $X$ are 
SoS certifiably bounded by $\momb$ (see \Cref{def:sos-certifiably-bounded-moments} for the formal definition).
\end{enumerate}
\end{assumption}
\noindent We note that assumptions 1 and 2 are common in the context of robust linear regression (see, e.g.,  \cite{DasJKS22,cherapanamjeri2020optimal}).
Assumption 3 is made so that the algorithm can take 
advantage of higher-order moment information from the 
distribution and is satisfied by a wide range of structured 
distributions, e.g., all strongly logconcave distributions.

Our main result is the following theorem.

\begin{theorem}[Main Algorithmic Result]\label{thm:main}
{
Let $\alpha \in (0, 1/2)$, $\sigma{>}0$, $k \in \mathbbm Z^+$ and $\beta^* \in \R^d$.
Assume that $\sigma{ \leq} R$, $\|\beta^*\|_2 {\leq} R$, and $k \leq \Delta / 2$. 
There is an algorithm that takes as input $\alpha, \sigma, R, k$, draws $m= \tilde O \lp( \lp((kd)^{O(k)} / \alpha  +   \alpha^{-3} \rp)
\log \lp(\frac{R}{\sigma}  \rp)  \rp)$
batches of size $n=O(k^2 Q^{2/k} \alpha^{-6/k})$ from the distribution of \Cref{def:glr}, 
and returns a list of estimates of size $O(\alpha^{-1})$, such that, with high probability, there exists at least one estimate $\hat \beta$ satisfying
$
\norm{\hat \beta - \beta^*}_2 = 
O \lp(  kQ^{1/k} \sigma \alpha^{-3/k}/\sqrt n
\rp)
$.}
\end{theorem}

\hide{\begin{remark}
We remark that when $R$ is large, there is a preprocessing step that can remove the $R$ dependency from our algorithm's runtime and sample complexity. 
We give the formal argument in [REF].
\end{remark}}

{Some remarks are in order. 
\Cref{thm:main} provides a substantial qualitative 
improvement over the bounds of \cite{DasJKS22} 
by succeeding for a dramatically smaller batch size 
while at the same time improving the estimation error. 
Concretely, our algorithm can use batch 
size of $n=O(k^2 \alpha^{-6/k})$ 
for any $k \in \Z_+$ of our choice, 
while \cite{DasJKS22} was only able to work with $n=\Omega(1/\alpha)$. 
We note that these improvements are possible due to our
stronger distributional assumptions that allow us to leverage
higher moments.

Conceptually, we view the capability of our algorithm to work with a \emph{flexible} batch size 
as a valuable feature---especially in real-world applications 
where the batch size corresponds to quantities 
that are not controllable by algorithm designers, 
i.e., the number of datapoints contributed by each provider.
Our result essentially shows that there 
is a smooth tradeoff between the batch size provided 
and the computational resources required.

More generally, our algorithm can cover the entire regime of 
$C \log^2(1/\alpha) \leq n \leq C/\alpha$, if we do not 
necessarily restrict $k$ to be an absolute constant. A limitation 
is that reaching the lower end of the regime 
would require $k$ to be  super-constant, 
namely $k \sim \log^2(1/\alpha)$, which would result in quasi-polynomial runtime. 
Interestingly, even for that lower regime of $n$, \Cref{thm:main} 
gives the first non-trivial (i.e., sub-exponential time) 
algorithm for the underlying task.
    
  Complementing our upper bounds, we point out that the super-polynomial dependence for  extremely small values of $n$ might be inherent. Via a simple reduction from the non-batch to the batch-setting (combined with the lower bound of \cite{DiaKPPS21}), we give evidence that the computational resources used in the algorithm of \Cref{thm:main} do not suffice for $n$ significantly smaller than $\log(1/\alpha)$. See Appendix~\ref{app:reduction} for the relevant discussion.
    
}

\subsection{Technical Overview} 
\label{subsec:our_techniques}

\paragraph{Prior Techniques}
We start by reviewing the algorithm of \cite{DasJKS22}, which uses 
a batch size of $n = \tilde{\Omega}(1/\alpha)$. 
For simplicity, consider the case where the covariates are standard normal. 
The high-level idea in \cite{DasJKS22} is to search for approximate stationary points\footnote{A vector $\beta \in \R^d$ is called a $\xi$-approximate stationary point of some function $f: \R^d \mapsto \R$ if it holds $\norm{ \nabla f(\beta) }_2 \leq \xi$. } 
of the $L_2$-loss, $f(\beta) = \frac{1}{2} \; \E_{ (X,y) } \lp[  ( \beta^\top X - y )^2 \rp].$
Without outliers, the expected gradient precisely equals $\beta - \beta^*$. 
If we recover a $\xi$-approximate stationary point for the inlier distribution,  
we can estimate $\beta^*$ up to an error of $O(\xi)$, given enough samples. 
With outliers, the method exploits an upper bound on the covariance of the gradient distribution  of the inliers to detect outliers and to control their influence. 
They then use the multi-filter approach for list-decodable estimation~\cite{DiaKK20-multifilter} 
to find a subset of the samples with the covariance matrix of the gradients being upper bounded by $O(1) \; \norm{\beta - \beta^*}^2_2 / n \vec I$, with an $\alpha$ overlap with the inliers. 
This ensures that a $\xi$-approximate stationary point under the corrupted sample distribution will still be a 
$(\xi + \norm{ \beta - \beta^* }_2 / \sqrt{n\alpha})$-approximate stationary point under the inlier distribution. This means that $\beta$ approximates $\beta^*$ up to an error of $\xi + \norm{ \beta - \beta^* }_2 / \sqrt{n\alpha}$.
Unfortunately, this results in  $\norm{ \beta - \beta^*}_2 \leq O(\xi)$ \emph{only when $n \gg 1/\alpha$}, since $\sqrt{\alpha n}$ needs to be larger than 1.

\medskip

\noindent In the remainder of this section, we outline the ideas of our approach.

\paragraph{Iterative Estimation of the Regressor}
Our main idea is to incorporate higher moment information into the estimator to alleviate the requirement on batch sizes.
There are two main challenges in exploiting higher moment information.
First, existing multi-filter approaches in the literature do not exploit higher moments, thus they are not easily modified. 
Second, existing higher-moment filters are not designed for iterative use.
They can generate a list of potential candidates but cannot progressively refine sample clusters for cleaner data segmentation, as the gradient descent method by \cite{DasJKS22} requires. 

To address this, our approach adopts a similar framework to~\cite{diakonikolas2019efficient} for robust linear regression with a small fraction of outliers, but in a non-batch setting. 
Here is an overview of their algorithm:
They begin by estimating the mean of the product of the covariate $X$ with the label $y$. 
In the outlier-free setting, this expectation equals to the true regressor $\beta^*$, and the covariance matrix can be bounded above by $O \lp( \norm{\beta^*}_2^2\rp) \vec I$ (assuming $\norm{\beta^*}_2 \gg \sigma$). 
They then use an algorithm for robust mean estimation for bounded-covariance distributions to derive an initial estimate $\hat \beta$ with a bounded error relative to $\beta^*$. 
They then improve the error by bootstrapping this approach. 
To do this, they adjust the labels via the transformation $y' = y - \hat \beta^\top X$. This reduces the problem of learning $\beta^*$ to another robust linear regression instance whose solution
has much smaller norm. Repeating this process iteratively allows them to refine their estimate to a final error of $O(\sqrt{\eps} \sigma)$.

In our setting, a natural strategy is to replace the robust mean estimation algorithm with one designed for list-decodable mean estimation, 
since the fraction of corruptions is larger than $1/2$.
Suppose that we have an algorithm $\mathcal A$ which 
produces a list of candidate regressors $\{\beta_i\}_{i=1}^{m}$ such that at least one of them satisfies 
$\norm{ \beta_i - \beta^* }_2 \leq \norm{\beta^*}_2 / 2$.
By applying the transformation $y' = y - \beta^\top_i X$, 
we can create $m$ distinct linear regression instances 
such that the regressor of one of these will have a significantly smaller norm. 
This allows us to compute more accurate estimates in the next iteration.

\paragraph{Beyond Second Moments}
As mentioned in the last paragraph above, the natural approach is to iteratively use a list-decodable mean estimation algorithm like the one from \cite{DiaKK20-multifilter} in order to estimate the mean of the random variable $W=\frac{1}{|B|}\sum_{(X,y) \in B} y X$ (where $B$ is an inlier batch) and reduce the error by a factor of 2 in each iteration. 
List-decodable mean estimators that rely only on second moment information have error behaving like $\sqrt{\|\cov(W)\|_{\op}}/\sqrt{\alpha}$, 
where $\sqrt{\|\cov(W)\|_{\op}} = O \lp( \norm{\beta^*}_2 / \sqrt{n}\rp) $ 
is the maximum standard deviation of $W$ along any direction. 
This already reveals the problem with this approach: 
for the error to become less than $\|\beta^*\|_2/2$, 
we need batch size $n \gg 1/\alpha$.

We overcome this (Proposition~\ref{cor:beta-estimate}) 
by using a list-decodable mean estimator that uses higher moment information, 
like Theorem 5.5 from \cite{kothari2017better}, or Theorem 6.17 from \cite{diakonikolas2023algorithmic}. 
These algorithms are based on the Sum-of-Squares hierarchy 
and their guarantee is that  whenever the higher moments of the inliers 
are ``SoS-certifiably bounded'' by $M$, then the estimation error is $O(M^{1/(2k)}\alpha^{-3/k})$.

However,
to leverage the above SoS-based algorithm, 
we require sharp SoS bounds on moments of
the batched regressor estimator $W = \frac{1}{|B|} \sum_{(X,y) \in B} y X$, 
which is a sum of i.i.d. random variables,
while our distributional assumption only posits that the covariate $X$ has certifiably bounded moments.
The fact that $W$ should also have bounded moments (but not necessarily SoS certifiable) follows 
from the famous Marcinkiewicz Zygmund Inequality. 
Unfortunately, to the best of our knowledge, 
an SoS proof of this inequality does not exist in the literature.
We give the first SoS proof of the inequality 
using combinatorial arguments (cf. Lemma~\ref{cor:sos-iid-moment}). 
We believe that this technical lemma may be of broader applicability.

Given the SoS moment bounds on $W$, 
the SoS-based list-decoding algorithm allows us to construct 
a list of size $O(1/\alpha)$ such that one of the estimates is 
$
O_k (  ( (\norm{\beta^*}_2^{2k}/n^{2k})^{1/(2k)} \alpha^{-3/k} )
$-close to $\beta^*$.
Hence, we will have some estimate 
$\beta_i$ satisfying $\norm{\beta_i - \beta^\ast}_2 \leq \norm{\beta^\ast}/2$
whenever $n \gg \alpha^{-3/k}$, which is a significant relaxation from the condition $n \gg 1/\alpha$ required 
by both the first approach and the approach of \cite{DasJKS22}.

\paragraph{List Size Pruning}
Having gotten the right estimate for one step, we bootstrap this to design an iterative algorithm such that the final list will contain an element that is sufficiently close. 
A significant challenge arises during the iterative phase of our list decoding algorithm. 
Initially, we generate a list of $O(1/\alpha)$ hypotheses, with the guarantee that at least one of them is near $\beta^*$. 
For each hypothesis, iterating further produces another $O(1/\alpha)$ hypotheses for each of the original hypotheses. 
Without careful management, this process can lead to an increase in the number of hypotheses that scales exponentially with the number of iterations, rendering the algorithm's complexity infeasible.

We overcome this (Proposition~\ref{prop:prune}) with techniques 
inspired by Theorem A.1 in \cite{DiaKK20-multifilter} (see also Exercise 5.1 in \cite{diakonikolas2023algorithmic}), 
which performs list-size reduction for list-decodable mean estimation. 
The general principle behind these methods is to  check whether each hypothesis in the list has a $\Theta(\alpha)$-fraction subset of the samples 
associated with it such that the hypothesis ``explains'' these samples. 
This results in certain ``consistency'' tests, on the basis 
of which we can prune elements of the list. 
The tests are designed such that 
(i) $\beta^*$ and the subset of inlier samples should pass the consistency tests, and that 
(ii) for any pair of sufficiently separated hypotheses, if they both pass the tests,  
their corresponding sets cannot have a large overlap. 
Given property (ii), the argument from \cite{DiaKK20-multifilter} shows that we can find a small cover of the set of plausible hypotheses.

In the list pruning step for list-decodable mean estimation, one usually leverages the fact that the inlier samples cluster closely around the learned mean (see, e.g., \cite{diakonikolas2018list, DiaKK20-multifilter}). This ensures survival of the optimal mean from the pruning procedure. 
One may want to generalize the test to the linear regression setting by asserting that $X y$ should concentrate around the candidate regressor $\beta$.
However, such a test turns out to be sub-optimal for linear regression \footnote{Intuitively, this is because such a test fails to take into account the 
influence of the size of $\beta$ on the concentration of the inlier samples.}. 
Instead, we design the following ``cross-candidate'' test: we keep $\beta$ only if 
$\beta$ demonstrates a smaller empirical $\ell_2$ error for an $\alpha$-fraction of 
selected batches in comparison to \emph{any other regressor significantly distant 
from $\beta$.} 
One may wonder whether the best regressor in the list can still survive the test, as 
there may be multiple equally good candidate regressors 
in the list with respect to the same cluster of batches. 
However, we note that the regressor is only compared to \emph{distant} regressors. 
Consequently, they must all be far from the optimal regressor (by the triangle 
inequality), ensuring the survival of the best regressor. 
This is formally shown in \Cref{lem:pruning-error-bound}.

\subsection{Related Work}
\label{subsec:related_work}
In this section we discuss related works from list-decodable linear regression and robust learning from batches. The problem of mixed linear regression is related very closely to our work as well, but due to space restrictions, we defer the relevant discussion to Appendix~\ref{app:mixed_ll}. 

\paragraph{List-decodable Linear Regression}
The list-decoding framework was first introduced in the context of machine learning in \cite{CSV17}. 
They derived the first polynomial time algorithm for list-decodable mean estimation
when the covariance is bounded. 
Later work considered the problem of list-decodable \emph{linear regression} in the non-batch setting \cite{KarKK19, RagYau19-list}. 
Unfortunately the runtime and sample complexity had an exponential dependence on $1/\alpha$,
this was later shown to be necessary for SQ algorithms \cite{DiaKPPS21}.

\paragraph{Robust Learning from Batches}
The problem of learning discrete distributions from untrusted batches was introduced 
in \cite{QiaoV18}, which gave exponential-time solutions. 
Progress was made by \cite{CheLS20} and \cite{jain2020optimal}, 
achieving quasi-polynomial and polynomial runtimes, respectively, with the latter also obtaining optimal sample complexity. 
Further developments by \cite{jain2021robust} and \cite{ChenLM20} expanded this work 
to one-dimensional structured distributions. 
\cite{DasJKS22} was the first to study the problem of list-decodable linear regression in the batch setting.
Compared to \cite{DasJKS22}, our method demonstrates substantial improvements in the error
and the required batch size, when the covariates are i.i.d. samples from $\cN(0, I)$. 
This can be attributed to our algorithm's ability to efficiently utilize higher moment information,
allowing for smaller batch sizes of $\Omega(k \alpha^{-6/k})$ and achieving an error of $O_{k, \sigma}(\sigma \alpha^{-3/k}/\sqrt{n})$, marking a significant improvement over a batch size of $\Omega(\alpha^{-1})$ and error of $O(\sigma/\sqrt{\alpha n})$, achieved in 
\cite{DasJKS22}.

\paragraph{Organization}
In Section~\ref{sec:prelims}, we define our notation and state some basic definitions about SoS programming. 
In Section~\ref{sec:SoS-alg}, we describe 
the main parts of our algorithm in Sections~\ref{subsec:generation},  and~\ref{subsec:pruning}; we then put things together to prove our main theorem in~\Cref{subsec:main}.

\section{Preliminaries}
\label{sec:prelims}
\noindent {\bf Notation}
We use $X \sim D$ to denote that a random variable $X$ is distributed according to the distribution $D$. 
We use $\cN(\mu, \Sigma)$ for the Gaussian distribution with mean $\mu$ and covariance matrix $ \Sigma$. 
For a set $S$, we use $X \sim S$ to denote that $X$ is distributed uniformly at random from $S$.
 We write $a \ll b$ to denote that $\alpha \leq c \cdot b$ for a sufficiently small  absolute constant $c>0$.
 We use $a(n) = O_k(b(n))$ to denote that there is a constant $C$ such that for all $n > C$, $a(n) \leq C_k \cdot b(n)$ for a constant $C_k$ that can arbitrarily depend on $k$.

\paragraph{Sum-of-Squares Preliminaries}
The following notation and preliminaries are specific to the SoS part of this paper.  We refer to \cite{BarakSteurerNotes} for a more complete treatment of the SoS framework.

\begin{definition}[Symbolic Polynomial]
	A degree-$k$ symbolic polynomial $p$ with input dimension $d$ is a collection of indeterminates $\widehat{p}(\alpha)$, 
	one for each multiset $\alpha \subseteq [d]$ of size at most $k$.
	We think of it as representing a 
 degree-$k$ polynomial $p \, : \, \R^d \rightarrow \R$ whose coefficients are themselves 
	indeterminates via $p(x) = \sum_{\alpha \subseteq [d], |\alpha| \leq k} \widehat{p}(\alpha) x^\alpha$.
\end{definition}

\begin{definition}[SoS Proof]\label{def:sos-proof}
Let $x_1,\ldots,x_n$ be indeterminates and $\mathcal A$ be a set of polynomial equalities 
$\{ p_1(x) = 0, \cdots, p_{w}(x) = 0 \}$.
An SoS proof of the inequality $r(x) \geq 0$ consists of two sets of polynomials $\{r_i(x)\}_{i \in [m]} \cup \{\bar r_i(x)\}_{i \in [w]}$ such that $r(x) = \sum_{i=1}^m r_i^2(x) + 
\sum_{i=1}^w p_i(x) \bar r_i(x)
$.
If the polynomials $\{r_i^2(x)\}_{i=1}^m  \cup 
\{\bar r_i(x) p_i(x)\}_{i=1}^w$ all have degree at most $K$, we say that this proof is of degree $K$ and write $
\mathcal A
\sststile{K}{} r(x) \geq 0$.  When we want to emphasize that $x$ is the indeterminate in a particular SoS proof, we write  
$\mathcal A \sststile{K}{x} r(x) \geq 0$.
When $\cA$ is empty, we omit it from the notation.
\end{definition}

\section{SoS Based Algorithm for List-Decodable Linear Regression with Batches}
\label{sec:SoS-alg}
Our algorithm iteratively updates a list of candidates, ensuring that, in every iteration, at least one candidate from the list is close to the target regressor.
It does so by iteratively applying two subroutines. 
In Subsection~\ref{subsec:generation}, we discuss a list-decoding subroutine that, given batch sample queries, generates a list of candidates containing some near-optimal 
regressor. In Subsection~\ref{subsec:pruning},
we discuss a pruning subroutine that ensures that the size of our list remains bounded. Finally, in Subsection~\ref{subsec:main} we combine these components into the main algorithm (\Cref{alg:batch-rl}) and prove our main theorem.

\subsection{Single Iteration: Approximate Estimation of \texorpdfstring{$\beta^*$}{Lg}}
\label{subsec:generation}
In this section, we construct an efficient SoS-based list-decoding algorithm to estimate the regressor $\beta^*$, assuming that $\norm{\beta^*}_2 \leq R$.
Specifically, this can be used to perform crude list-decodable estimation of the optimal regressor. 

In the final algorithm, we will bootstrap this method to generate our final list with improved error guarantee.

\begin{proposition}
\label{cor:beta-estimate}
Let $\alpha \in (0, 1/2)$, $\delta \in (0,1)$, $m,n,k \in \mathbbm Z_+$, $\sigma, R > 0$, $\beta^* \in \R^d$.
Assume 
$\norm{\beta^*}_2 \leq R$ and $k \leq \Delta/2$.
Then, there exists an algorithm that takes 
$\alpha, k, \delta,\sigma, R$ in the inputs, it draws 
$m = O\big(  \batchsc \big) \alpha^{-1} \log(1/\delta)$ many batches  from the corrupted batch distribution of \Cref{def:glr}, 
runs in time $\poly( d^{k} m )$, and
outputs $O( \log(1/\delta) \alpha^{-1})$ many estimations such that there exists at least one estimation $\hat \beta$ satisfying
$\norm{\hat \beta - \beta^*}_2 \leq 
O \lp(  \sosbetagap \rp)$ with probability at least $1-\delta$ over the randomness of the batches.
\end{proposition}

A standard way of estimating the regressor is to consider the random variable $y X$.
When there are no outliers, $y X$ gives an unbiased estimator of $\beta^*$. 
In the batch setting, a natural estimator is to use the batch average $Z_B:= \frac{1}{n}\sum_{(X,y) \sim B}
y X$.
The main idea behind \Cref{cor:beta-estimate} is that we can leverage the property that the batch average estimator $Z_B$ has \emph{SoS-certifiably bounded central moments} when $B$ consists of \iid samples from the uncorrupted linear regression distribution $D_{\beta^*}$ (cf. \Cref{def:glr}).
\begin{definition}[SoS-Certifiably Bounded Central Moments]
\label{def:sos-certifiably-bounded-moments}
Let $M > 0$, $k$ be an even integer, and $D$ be a distribution with mean $\mu$. 
We say that $D$ has $(M,k,K)$-certifiably bounded moments
if 
$
\constrain \sststile{K}{v} \E_{X \sim D} \big[ 
( v^{\top} ( X - \mu ))^{k}
\big] \leq M.
$
We say a set of points $T$ has $2k$-th central moments SoS-certifiably bounded by $M$ if the empirical distribution over these points does so.
\end{definition} 

Observe that $Z_B$ is the sum of $n$ \iid copies of $yX$, and $X$ has SoS-certifably bounded moments by \Cref{ass:dist}.

Applying the Marcinkiewicz-Zygmund Inequality, which controls the moments of \iid random variables by their individual moments, will almost immediately yield that $Z_B$ also has bounded central moments. 
To further show that the bound is SoS-certifiable, we thus require an SoS version of this moment inequality, which is provided below.
\begin{restatable}[SoS Marcinkiewicz-Zygmund Inequality]{lemma}{MZInequality}
\label{cor:sos-iid-moment}
Let $v \in \R^d$, $X_1, \cdots, X_n$ be \iid random real vectors in $\R^d$, and $p: \R^d \times \R^d \mapsto \R$ be a degree-$t$ polynomial.
Assume that 
$$
\constrain
\sststile{ 2k t }{ v }
\E\lp[ ( p(v, X_i) - \E[p(v, X_i)] )^k \rp] \leq M
$$ for some number $M > 0$.
Then the degree-$k$ central moment of the sum of $p(v, X_i)$ is also SoS-certifiably bounded:
\begin{align}
\label{eq:sos-iid-moment}
&\constrain
\sststile{ k t }{ v }
\E \lp[ \lp(  \sum_{i=1}^n \lp( p(v, X_i) - \E[p(v, X_i)] \rp) \rp) ^k \rp] \nonumber \\
&\leq (kn)^{k/2} \;  M.
\end{align}    
\end{restatable}

Combining the above SoS inequality with the fact that
$yX = \lp( \beta*^\top X + \xi \rp) X$ is a degree-$2$ polynomial in $X$, which has SoS certifiably bounded moments, and $\xi$, which follows the Gaussian distribution, then gives essentially a population version of the moment bound.

The SoS moment bound on the empirical distribution over samples then follows by a careful analysis on the concentration properties of the empirical moments of $Z_B$.
See Appendix~\ref{sec:certifiably_bounded} for the detailed argument.
\begin{restatable}[SoS Moment Bound]{lemma}{SOSMOMENT}
\label{lem:sos-moment-bound}
Let $\alpha \in (0, 1/2)$, $\sigma{>}0$, $k \in \mathbbm Z^+$, $\beta^* \in \R^d$. 
Let $T$ be a set of $m$ batches drawn according to the distribution $D_{\beta^*}$ defined in \Cref{def:glr}, 
and batch size $n$. 
Assume that the clean covariates distribution $X$ satisfies \Cref{ass:dist} and $k \leq \Delta/2$.
Define $Z_B = \frac{1}{n} \sum_{  (X,y) \in B } X y$.
Suppose $m \gg \lp( \batchsc \rp) \alpha^{-1}$.
Then the following holds with probability at least $0.9$:
(a) $\{ Z_B \mid B \in T \}$ has 
$(M, 2k, 4k)$-certifiably bounded moments for some  
$M = \Mbound$, and 
(b) $ \cov_{ B \sim T }[ Z_B ] \preceq O (  (\normt{\beta^*}^2 + \sigma^2)/n ) \vec I $.
\end{restatable}
Once we have that the central moments of $Z_B$ are certifiably bounded, \Cref{cor:beta-estimate} follows from the following SoS-based list-decodable mean-estimation algorithm:
\begin{lemma}[Theorem 5.5 from \cite{kothari2017better}]
\label{lem:sos-filter}
Let $S$ be a set of points in $\R^d$ containing a subset $S_{\text{good}}$ with $|S_{\text{good}}| \geq \alpha |S|$. Moreover, assume that 
$S_{\text{good}}$ 
has $(M, 2k, K)$-certifiably bounded moments
for some positive integers $k, K$ and $M > 0$. 
Then there exists an algorithm that, given $S, k, K, M$ and $\alpha$, runs in time $\poly( d^{K}, |S| )$, and with probability $0.9$, returns a list of $O(1/\alpha)$ many vectors containing some $\hat \mu$ with $\norm{ \hat \mu - \mu_{ S_{\text{good}} } }_2 = O \lp( M^{ 1/(2k) } \alpha^{-3/k} \rp)$.
\end{lemma}

\begin{proof}[Proof of \Cref{cor:beta-estimate}]
Suppose we take  
$m \gg \lp( \batchsc \rp) \alpha^{-1}$ many batches.
Then  $m \; \Omega(\alpha) \gg \batchsc $ many of these batches are of inlier type with high constant probability. We denote the set of these batches by $G$.
Define $Z_B = \frac{1}{n} \sum_{(X,y) \in B} X y$.
Let $D_{\beta^*}^{\otimes n}$ be the distribution of a clean batch of size $n$ whose samples are all \iid from $D_{\beta^*}$ (cf. \Cref{def:glr}).
As shown in the proof of Lemma~\ref{lem:sos-moment-bound}, we have $\E_{ B \sim D_{\beta^*}^{\otimes n} }[Z_B] = \beta^*$ and $\cov_{ B \sim D_{\beta^*}^{\otimes n} }[Z_B]
\preceq O\lp({\lp(\sigma^2 + R^2\rp)}/{n}\rp) \vec I.$
Since $|G| \gg \batchsc$, by Markov's inequality,
it holds that
\begin{align}
\label{eq:vector-bound}
\left \|\littlesum_{B \in G} Z_B/ |G| - \beta^* \right \|_2
\leq O((\sigma + R)/\sqrt{n})
\end{align}
with high constant probability.
Besides, since $|G| \gg \batchsc$,
Lemma~\ref{lem:sos-moment-bound} shows that $
\mathcal Z :=\{Z_B \mid B \in G\}$ has $2k$-th central moments
SoS-certifiably bounded by 
\begin{align}
\label{eq:k-moment-bound}
M = \Mbound
\end{align}
with high constant probability.
Moreover, the covariance of $\mathcal Z$ can be bounded from above by 
\begin{align}
\label{eq:cov-bound}
\cov_{ Z \sim \mathcal Z } [Z] 
\preceq O \lp((\sigma^2 + R^2)/n \rp) \vec I
\end{align}
with high constant probability.
By the union bound, \Cref{eq:vector-bound}, \Cref{eq:k-moment-bound}, and \Cref{eq:cov-bound} hold simultaneously with high constant probability.
Conditioned on that, Lemma~\ref{lem:sos-filter} thus allows us to estimate the mean of 
$\{Z_B \mid B \in G\}$ up to accuracy 
$
O \lp( (k/\sqrt{n})  \; Q^{1/(2k)}
\lp( \sigma + R \rp)  \; \alpha^{-3k}
\rp).
$
Our estimate is then $
O \lp( (k/\sqrt{n})  \; Q^{1/(2k)}
\lp( \sigma + R \rp)  \; \alpha^{-3k}
\rp)
$ close to $\beta^*$ by \Cref{eq:vector-bound}, and
the triangle inequality.
This concludes the proof of \Cref{cor:beta-estimate}. Finally we can boost the probability of success to $1-\delta$ by running the above procedure $\log(1/\delta)$ many times and combining the lists obtained.
\end{proof}

\subsection{Pruning Routine}
\label{subsec:pruning}
In this subsection, we show that there is an algorithm,~\textsf{Pruning}, which reads a list $L$ containing a candidate close to $\beta^*$, and returns a sub-list $L' \subseteq L$ of size $O(1/\alpha)$  also containing a candidate close to $\beta^*$.  
\begin{restatable}[Pruning Lemma]{proposition}{PRUNE}
\label{prop:prune}
Let $\alpha \in (0, 1/2)$, $\delta \in (0,1)$, $k,n \in \mathbbm Z_+$, $\sigma, R > 0$, and $\beta^* \in \R^d$.
Let $L \subset \R^d$ be a list of candidate regressors, and $\beta \in L$ be a regressor such that $\norm{\beta - \beta^*}_2 < R$. 
Assume that the batch size $n$ satisfies that
$n \gg \prunebs$ and $k \leq \Delta / 2$.
Then there exists an algorithm \textsf{Pruning} that takes the list $L$, 
and the numbers $\alpha, \delta, R$ as input, 
draws $m = O \lp( \min \lp(\log(|L|), d^2\rp) \; \log(1/\delta) \;  \alpha^{-3} \rp)$ many batches from the corrupted batch distribution of \Cref{def:glr},
runs in time $\poly(d m |L|)$,
and outputs at most $O(1/\alpha)$ candidate regressors $L' \subseteq L$ such that there is at least one regressor $\beta \in L'$ satisfying $\norm{\beta - \beta^*}_2^2 \leq O\lp( \betagap \rp)$ with probability at least $1 - \delta$ over the randomness of the batches drawn.
\end{restatable}

The \textsf{Pruning} algorithm involves two phases.
Initially, it filters regressors $\beta \in L$ by keeping those matching a certain set of solvable linear inequalities. 
Then it selects a subset of the remaining regressors, ensuring each pair is adequately distant. 
Lemmas \ref{lem:pruning-list-size-bound} and \ref{lem:pruning-error-bound} respectively
show that the refined list is not excessively large
 and that it contains a vector near the true regressor $\beta^*$, if such a candidate exists in the original list $L$. 
The proof of Proposition~\ref{prop:prune} follows from the above two lemmas. \looseness=-1

For each regressor, we now describe the set of linear inequalities used in the pruning process involving a \emph{weighting function} $\mathcal W$ over the set of batches $T$. 
At a high level, a weighting function can be interpreted as a ``soft cluster'' for each candidate regressor $\beta$, and the inequalities aim to identify a soft cluster for each candidate regressor $\beta$ by ensuring:
(i) at least an $\alpha$-fraction of batches are included in the cluster, and 
(ii) $\beta$ has a smaller empirical $\ell_2$ error in comparison to any other regressor $\beta'$ that is significantly distant from $\beta$, based on the following conditions involving the constant $c$, radius $R$, standard deviation $\sigma$, and batch size $n$.
We denote this set of linear inequalities by $\IE(\beta; L, T, R)$, i.e.,  $\IE(\beta; L, T, R)$ is the following set of inequalities in the 
variable(s) $\mathcal W: T \mapsto [0, 1]$:
\begin{align}
&\textstyle\sum_{B \in T}  \mathcal W(B) \geq  0.9 \alpha |T|, \label{Eq2Cond0}
\\
     &\forall \beta' \in L \text{ satisfying }
     \norm{\beta' - \beta}_2 \geq c \lp( \shortbetagap \rp) \nonumber \\
     &\text{ for some sufficiently large constant } c
     \, , \nonumber \\
     &\textstyle \sum_{B \in T}
     \mathbbm 1 \{ 
     \textstyle\sum_{ (X,y)\in B } \lp( y-X^\top \beta\rp)^2  \nonumber \\
     &\leq 
     \textstyle\sum_{ (X,y)\in B } \lp( y-X^\top \beta'\rp)^2 
 \}
     ~\mathcal W(B)
     \leq \frac{\alpha}{20} \textstyle\sum_{B \in T}  \mathcal W(B) \;. \label{Eq2Cond1}
\end{align}
We now provide some intuition about why there cannot be too many regressors whose associated linear inequalities are 
satisfiable subject to the constraint that they are all sufficiently separated.
At a high level, this is because condition (ii) enforces the soft clusters associated with two sufficiently separated candidate regressors must have small intersection as two candidate regressors cannot \emph{simultaneously} do better than the other in terms of their empirical errors on \emph{the same batch}.
We now precisely state the lemmas. For proofs, please see Appendix~\ref{app:pruning}. 
\begin{restatable}[List Size Bound]{lemma}{LISTSIZE}
\label{lem:pruning-list-size-bound}
Let $R > 0$, and
$L$ be a list of candidate regressors.
Let $T$ be a set of batches.
Let $L' \subseteq L$ be a sublist of candidate regressors satisfying the following conditions: (1)
$\IE(\beta; L, T, R)$ has solutions for each $\beta \in L'$, and 
(2) $\norm{\beta_1 - \beta_2}_2 \geq c \lp( \betagap \rp)$ for any two $\beta_1, \beta_2 \in L'$.
Then it holds the size of $L'$ is at most $O(1/\alpha)$.
\end{restatable}
The next lemma we need shows that the list after an application of Lemma~\ref{lem:pruning-list-size-bound} contains an element close to $\beta^*$.
To get some intuition, let us fix some $\beta$ that is close to $\beta^*$ and some $\beta'$ that is far from $\beta$.
By the triangle inequality, $\beta'$ therefore should be far from $\beta^*$ as well.
As the square loss of a candidate regressor can be viewed as a surrogate for the distance between the regressor and the optimal, it follows that $\sum_{ (X,y) \in B }  \lp( y - X^T \beta \rp)^2$ must be significantly less than $\sum_{ (X,y) \in B }  \lp( y - X^T \beta' \rp)^2$ over all inlier batches in expectation.
Due to $L_2$-$L_4$ hypercontractivity of $X$, we can show that $\lp( y - X^T \beta' \rp)^2$ must be weakly anti-concentrated. 

On the other hand, due to the bounds on the higher-order moments of $X$, we can show that $\lp( y - X^T \beta \rp)^2$ must be sufficiently concentrated around its mean.
Combining the two observations then we show that \Cref{Eq2Cond1} must hold with high probability over the inlier distributions. 
Therefore, setting $\mathcal W$ to be the indicator variables for inlier batches must constitute a valid solution to the set of inequalities $\IE(\beta;L,T,R)$ constructed.

Formally, we establish the following lemma, whose
full proof can be found in Appendix~\ref{app:pruning}.
\begin{restatable}[Error Bound]{lemma}{ERROR}
\label{lem:pruning-error-bound}
Let $\alpha \in (0, 1/2)$, $\delta \in (0,1)$, $n, K \in \mathbbm Z_+$, $\sigma, R > 0$, and $\beta^* \in \R^d$.
Let $L$ be a list of candidate regressors of size $K$
, and $\beta \in L$ be a regressor such that $\normt{\beta -  \beta^*} < R$.
Let $n \gg \prunebs$ be the batch size parameter.
Suppose $T$ is a set of  $m  \gg \min \lp(\log(K), d^2\rp) \; \log(1/\delta) \;  \alpha^{-3}$ many batches of size $n$ drawn from the corrupted batch distribution of \Cref{def:glr}.
With probability at least $1 - \delta$ over the randomness of $T$, we have that 
the system $\IE(\beta; L, T, R) $ has solutions.
\end{restatable}
Proposition~\ref{prop:prune} now follows by an application of the above lemmas.
\begin{proof}[Proof of Proposition~\ref{prop:prune}]
The pruning procedure proceeds in two steps.
First, it filters out the $\beta$ such that $\IE(\beta; L, T, R)$ has no solution. 
Second, among the remaining regressors, it adds them into the output list as long as it is not 

$ c \; \lp( \betagap \rp)$
-close to some existing regressor in the output list for sufficiently large constant $c$.
Then the size of the output list is at most $O(1/\alpha)$ by Lemma~\ref{lem:pruning-list-size-bound}.

\noindent It remains to show that there exists some $\beta'$ close to $\beta^*$ in the output list if there exists some $\beta$ satisfying $\norm{\beta - \beta^*}_2 < R$ in the input list.
By Lemma~\ref{lem:pruning-error-bound}, $\beta$ will not be filtered out with probability at least $1 - \delta$ in the first step.
In the second step, either $\beta$ is added to the output list or there must be some $\beta'$ satisfying $\norm{\beta' - \beta}_2 \leq O\lp( \betagap\rp)$. 
Hence, we must have $
\norm{\beta' - \beta^*}_2 \leq
O\lp( \betagap \rp)$ by the triangle inequality.
This concludes the proof. %
\end{proof}

\subsection{Putting Things Together}
\label{subsec:main}

\begin{algorithm}[]
\caption{Batch-List-Decode-LRegression (Informal)} 
\label{alg:batch-rl-informal}
\begin{algorithmic}[1]
\State \textbf{Input}: Batch sample access to the linear regression instance, $\sigma, R$ as specified in \Cref{thm:main}.
\State Initialize a list $L = \{ 0 \}$ to hold candidate regressors.

\For{$t = 0, \cdots, \log(R/\sigma )$}
    \State Initialize an empty list $L'$ to store refined candidate regressors.
    \For{candidate regressor $\hat \beta \in L$}
        \State Take sufficiently many batched samples $T$.
        \State For each $(X,y)$ in $T$, compute the residue $(X,y {-} {\hat \beta}^{\top} X)$. 
        \State Denote the resulting new set of batches as $T'$.
        \State Learn a new list of regressors by running the algorithm from \Cref{cor:beta-estimate} on $T'$.
        \State Add the results to $L'$.
    \EndFor
    \State Replace $L$ with $L'$.
    \State Run algorithm from \Cref{prop:prune} to prune the list $L$ into one with size $O(1/\alpha)$.
\EndFor

\State \textbf{return} $L$.
\end{algorithmic}
\end{algorithm}

In this section, we prove our main theorem, 
starting with a high-level overview of the algorithm, 
which mirrors the structure of the robust linear regression algorithm 
from \cite{diakonikolas2019efficient} that can tolerate a small constant fraction of outliers. 
The process begins with estimating $\beta^*$, as outlined in Corollary~\ref{cor:beta-estimate}, 
to obtain a list $L$. 
We then create new linear regression instances by transforming each sample $(X,y)$ 
into $(X, y - \beta^{\top} X)$ for every $\beta$ in $L$. 
This ensures that for at least one transformed instance, the norm of the optimal regressor decreases significantly. 
Applying Corollary~\ref{cor:beta-estimate} to these instances and merging the resulting 
lists yields a list containing a candidate regressor that is closer to $\beta^*$. 
We iterate this process until we get a list with an element that is sufficiently close to $\beta^*$.
One issue is that the list size will increase exponentially in terms of the number of iterations. 
To counter this, Proposition~\ref{prop:prune} is employed to prune the list to an optimal size while maintaining the error of the best candidate regressor to within a constant factor. 
The pseudocode of an informal version of the algorithm is provided in \Cref{alg:batch-rl-informal}.
See \Cref{alg:batch-rl} in the appendix for the formal version.

\begin{proof}[Proof of \Cref{thm:main}]
Suppose $\norm{\beta^*}_2 \leq R$.
We claim that with high constant probability the list $L$ will include a candidate $\beta$ such that $\norm{\beta - \beta^*}_2 \leq O(\sigma)$ after all but the last iteration with respect to $t$.
This is trivially true if $R = O(\sigma)$.
Otherwise, if $R \gg \sigma$, we use induction to argue that
for all $t = 0, \cdots, \max( \log(c_0R/\sigma), 0) - 1$, 
with probability at least $\lp( 1 - 2 \tau \rp)^t$, 
the list $L$ will contain some candidate regressor $\beta$ such that $\norm{\beta - \beta^*}_2 \leq R \; 2^{-t}$ after the $t$-th iteration.
Note that this implies the above claim for $t = \max( \log(c_0R/\sigma), 0) - 1$, and for all $t$ considered in the inductive hypothesis we have $\sigma \ll R 2^{-t}$.

Conditioned on the existence of such a $\beta$ in the list after the $(t-1)$-th round.
By \Cref{cor:beta-estimate},
when we execute line~\ref{line:sos-learn} in the iteration where $\hat \beta = \beta^{(1)}$, 
we obtain a list of regressors such that with probability at least 
$1 - \tau$ there exists some $\beta^{(2)} \in L_0$ and some small constant $c$ such that 
\begin{align*}
&\normt{ \beta^{(2)} + \beta^{(1)} - \beta } \\
&\leq 
O \lp( \frac{(k \; Q^{1/(2k)} \; \lp( 
R \; 2^{-t+1} + \sigma
\rp) \; \alpha^{-3/k})}{\sqrt n}
\rp)
\leq c \; R \; 2^{-t} 
\, ,    
\end{align*}
where the last inequality is true as long as $n \gg \; k^2 \; \alpha^{-6/k} \; Q^{1/k}$, and $R \; 2^{-t} \gg \sigma$.
The list $L_0$ is now of size $O(\log(1/\tau) / \alpha)$.
Thus, after Line~\ref{line:prune}, by Proposition~\ref{prop:prune}, the list $L$ gets pruned into one with size $O(1/\alpha)$ with probability at least $1 - \tau$.
Moreover, $L$ still contains some candidate $\beta^{(3)}$ such that 
$$
\normt{ \beta^{(3)} - \beta }
\leq O \lp( c \; R \; 2^{-t} + k \alpha^{-1/k} \sigma Q^{1/k} / \sqrt{n} \rp) < R \; 2^{-t} \, ,
$$
as long as $c$ is sufficiently small, $n \gg k^2 \alpha^{-2/k} Q^{2/k}$, and $R 2^{-t} \gg \sigma$.
This concludes the induction.

In the last round, since we have $R \leq O(\sigma)$, with a similar argument, 
we arrive at a list of size $O(1/\alpha)$ that contains some $\beta$ satisfying
\begin{align*}
&\norm{\beta - \beta^*}_2
\leq
O \bigg(  (k\; Q^{1/(2k)} \; \sigma \; \alpha^{-3/k}) /\sqrt n  \\
&+ (k \alpha^{-1/k} \sigma Q^{1/k})/\sqrt{n}
\bigg)
= O \lp( (k/\sqrt n) \; Q^{1/k} \; \sigma \; \alpha^{-3/k}
\rp).    
\end{align*}
It is not hard to see that the total number of samples consumed by the algorithm is at most
\begin{align*}
&O \lp(  \frac{(4kd)^{8k} Q^{-1} }{\alpha}
+ \frac{\log\lp( \log(1/\tau) / \alpha \rp)}{ \alpha^3} \rp) \; \log(1/\tau)~\log(R/\sigma)\\
&= \tilde O \lp( \lp((4kd)^{8k} Q^{-1} / \alpha  +   \alpha^{-3} \rp)
\log \lp(R/\sigma  \rp)  \rp).
\end{align*}
Moreover, the runtime is polynomial in the sample size times $d^{2k}$, 
which is the space complexity required for representing a moment $2k$ tensor.
\end{proof}

\bibliographystyle{alpha}
\bibliography{newrefs}

\appendix
\newpage
\section*{Appendix}  

\paragraph{Organization}
In Appendix~\ref{app:mixed_ll} we discuss work on the problem of mixed linear regression, which is very related to the setting we consider here. 
In Appendix~\ref{app:sos_facts} we state some basic SoS facts. Then, in Appendix~\ref{sec:certifiably_bounded}, we adapt proofs from \cite{DKKPP22} to show that if the original distribution has certifiably bounded moments, then the uniform distribution over a sufficiently large sample also has certifiably bounded moments. We then use this to prove Lemma~\ref{lem:sos-moment-bound}.
In Appendix~\ref{app:pruning} we prove the lemmas required to get our pruning guarantee in Subsection~\ref{subsec:pruning}. 
Finally, in Appendix~\ref{app:reduction} we present a simple reduction of our problem to the problem of list-decodable linear regression in the non-batch setting. 

\section{Related Work on Mixed Linear Regression}
\label{app:mixed_ll}
The mixed linear regression setting is when the data is generated by a mixture of $t$ distributions $D_1, \dots, D_t$, each on $\R^{d} \times \R$ such that $(X, y) \sim D_i$ is equivalent to $y = \beta_i^T X + \xi$ for $X \sim \cN(0, I)$ and $\xi \sim \cN(0, \sigma^2)$ \cite{Deveaux89,JorJac94}. We refer the reader to Section 1.2 of ~\cite{CheLS20} for a detailed summary of prior work for this problem. In the non-batch setting, this problem suffers from an exponential dependence on $t$. This is inherent in moment-based approaches, as shown in~\cite{CheLS20}. The most efficient algorithm for the problem is due to~\cite{DiaKan20-small-cover} which runs in time and needs samples quasi-polynomial in $t$. More recently,~\cite{DK24-implicit} gave a fully polynomial-time learner, for positive constant $\sigma$, that performs density estimation in total variation distance.

In the batch setting, this was first studied for covariates drawn from $\cN(0, I)$ by \cite{KonSKO20,kong2020meta}.
Here, all the samples from each batch belong to a single component. 
\cite{kong2020meta} design an algorithm that requires $O(d)$ batches of size $O(\sqrt{t})$ to solve the problem efficiently (including in terms of the parameter $t$). Subsequently, \cite{KonSKO20} uses the sum-of-squares hierarchy to design a class of algorithms that can trade between the batch size and sample complexity while being robust to a small fraction of outliers.\footnote{This is similar in flavor to what we do in this paper but for the much harder problem of list-decodable linear regression.} Finally, \cite{jain2023linear} greatly generalize the scope by designing an algorithm that can recover the regressors for all components such that at least an $\alpha$ fraction of the batches satisfy a linear-regression model with variance in the noise bounded by $\sigma^2$. Their algorithm works even when the covariates for each component are different, varying, and heavy-tailed. They do this by allowing for batches of nonuniform size. They require $\tilde O(d/\alpha^2)$ batches of size $\geq 2$ and  $\tilde\Omega\min(\sqrt{t},1/\sqrt{\alpha})/\alpha$ batches of size $\tilde\Omega\min(\sqrt{t},1/\sqrt{\alpha})$.  This is very close to the list-decodable setting we study in this paper; however, we do not allow for nonuniform batch sizes. Even so, our algorithm can improve the batch size required by a constant power of the algorithm designer's choice in the exponent.

\section{Further Background on SoS Proofs and Moment Bounds}
\label{app:sos_facts}
It is a standard fact that several commonly used inequalities like the triangle inequality, Cauchy-Schwartz, or AM-GM inequalities have an SoS version.

\begin{fact}[SoS Cauchy-Schwartz and H\"older (see, e.g., \cite{hopkins2018clustering})]\label{fact:sos-holder}
	Let $f_1,g_1,  \ldots, f_n, g_n$ be indeterminates.
	Then, 
	\begin{align*}
	\sststile{2}{f_1, \ldots, f_n,g_1, \ldots, g_n} \Set{ \Paren{\frac{1}{n} \sum_{i=1}^n f_i g_i }^{2} \leq \Paren{\frac{1}{n} \sum_{i=1}^n f_i^2} \Paren{\frac{1}{n} \sum_{i=1}^n g_i^2} } \;.
	\end{align*} 
\end{fact}

\begin{fact}[SoS Triangle Inequality]\label{fact:sos-triangle}
	If $k$ is an even integer, 
	$\sststile{k}{a_1, a_2, \ldots, a_n} \Set{ \left(\sum_{i=1}^n a_i \right)^k \leq n^k \Paren{\sum_{i=1}^n a_i^k} }.$ 
\end{fact} 

\begin{fact}[SoS AM-GM Inequality, see, e.g., Chapter 2 of \cite{hardy1952inequalities}] \label{fact:sos-am-gm}
Let $k$ be an even integer, and $\{ w_i \}_{i=1}^n$ be integers such that $\sum_{i=1}^n w_i = k$. Then it holds that
$$
\sststile{k}{ x_1. \ldots, x_n } 
\lp \{ \prod_i x_i^{w_i} \leq \sum_{i=1}^n \frac{w_i}{k} x_i^k \rp\}.
$$
\end{fact}
Using these inequalities, we can construct SoS proofs for bounds of moments of sum of \iid random variables with SoS certifiably bounded moments. The non-sos version of the inequality is commonly known as the Marcinkiewicz-Zygmund inequality.
\begin{proof}[Proof of \Cref{cor:sos-iid-moment}]
For notational convenience, we define $y_i = p( v, X_i ) - \E[ p(v, X_i) ]$.  
Note that each $y_i$ is a degree-$t$ polynomial in $v$ and $X_i$.
If we expand $\left( \sum_{i=1}^n y_i \right)^k$, we get $n^k$ many monomials of the form $ \prod_{j=1}^k y_{ \sigma_j } $ for some $\sigma_j \in [n]^k$.
If the degree of some $y_i$ is $1$, the expected value of that monomial will be $0$ since 
$\E[ y_i ] = \E[ p( v, X_i ) - \E[ p(v, X_i) ] ] = 0 $.
Hence, $ \E [ \prod_{j=1}^k y_{ \sigma_j } ]$ is non-zero only if the number of variables appeared is at most $k/2$ since otherwise some $y_i$ must have degree-$1$ by the pigeonhole principle.
By a simple counting argument, we have that the number of monomials with non-zero expectations is then at most $ { n \choose (k/2) } \; k^{k/2}.$
Let $ \prod_{i=1}^n y_i^{w_i} $ be one of such monomial with non-zero expectation, where $\sum_{i=1}^n w_i = k$. 
We can bound its expectation from above by
$$
\constrain \sststile{v}{2kt}
\E\lp[ 
\prod_{i=1}^n y_i^{w_i}
\rp]
\leq
\sum_{i=1}^n
\frac{w_i}{k}
\E\lp[ 
y_i^{k}
\rp]
\leq M \, ,
$$
where the first inequality is by \Cref{fact:sos-am-gm}, and the second inequality is by our assumption that 
$ \constrain \sststile{v}{2kt} \E [y_i^k]
= \E\lp[ \lp( p( v, X_i ) - \E[ p(v, X_i) ] \rp)^k  \rp] \leq M$.
Since there are at most $n^{k/2}  \; k^{k/2} $ such monomials with non-zero expectation, 
it then follows that 
$$
\constrain \sststile{v}{2kt}
\E \lp[  \lp(\sum_{i=1}^n y_i \rp)^k  \rp]
\leq (kn)^{k/2} \; M.
$$
\end{proof}

\newpage

\section{Algorithm Pseudocode}

\begin{algorithm}[]
\caption{Batch-List-Decode-LRegression} 
\label{alg:batch-rl}
\begin{algorithmic}[1]
\State \textbf{Input}: Batch sample access to the linear regression instance, and $\alpha, \sigma, R, k$ as specified in \Cref{thm:main}.
\State Initialize $L = \{ 0 \}$.
\State Set failure probability $\tau = 0.001/\log(R/\sigma)$.
\State Let $c_0$ be some sufficiently small constant and $C$ be some sufficiently large constant.

\For{$t = 0, \cdots, \max(\log( c_0 R/\sigma ), 0)$}
    \State Initialize $L_0 = \{0\}$
    \For{candidate regressor $\hat \beta \in L$}
        \For{$r = 0, \cdots, \log(1 / \tau)$}
            \State Take a batch of $C \; (2dk)^{10k} / \alpha$ samples $T$.
            \State For each $(X,y)$, compute $(X,y {-} {\hat \beta}^{\top} X)$. 
            \State Denote the new set of batches as $T'$.
            \State Learn a list $L_1$ of regressors by running the algorithm from \Cref{cor:beta-estimate} on $T'$. \label{line:sos-learn}
            \State Add the candidate regressors $\{\hat \beta' + \hat \beta \mid \hat \beta' \in L_1\}$ into $L_0$.
        \EndFor
    \EndFor
    \State Set $L = L_0$.
    \State Draw $ C \; \min \lp(\log( \log(1/\tau) / \alpha ), d^2 \rp) \log(1/\tau) \alpha^{-3}$ batch of samples $T'$.
    \State Run algorithm from \Cref{prop:prune} on $T'$ to prune the list $L$ with failure probability $\tau$. \label{line:prune}    
\EndFor

\State \textbf{return} $L$.
\end{algorithmic}
\end{algorithm}

\section{Certifiably Bounded Moments of the Regressor Estimator}
\label{sec:certifiably_bounded}
In this subsection, we give the proof of \Cref{lem:pop-to-empirical}. 
We first give several preliminary lemmas regarding the concentration properties of empirical higher order moment tensors of distribution with bounded central moments.
The proof is similar, for example, to Lemma A.4 from \cite{DiaKKPP22-neurips}.

\begin{lemma}\label{lem:basic_linf_consc-full}
Let $D$ be a distribution over $\R^d$ with mean $\mu$ and $t \in \mathbb Z_+$. 
Suppose that $D$ has its covariance bounded from above by $\kappa I$, and its degree-$2t$ central moments bounded by $F$,
i.e., $ \E_{ X \sim D } \lp[ \lp| \vec v^T  (X - \mu) \rp|^{2t} \rp] \leq F $.
Let $X_1, \dots, X_m$ be $m$ i.i.d.\ samples from $D$.
The following inequalities hold with high constant probability.
\begin{align*}
\left \| \E_{i \sim [m]}[(X_i - \mu)^{\otimes t}] -  \E_{X \sim D}[(X - \mu)^{\otimes t}] \right \|_{\infty} \leq 
O \lp( d^{t} \sqrt{t F / m } \rp).
\end{align*}	
Define $\overline \mu = \frac{1}{m} \sum_{i=1}^m X_i$.
We also have that
\begin{align*}
    \| \mu - \overline{\mu} \|_2
    \leq  O\lp( \sqrt{ \kappa d / m} \rp).
\end{align*}
\end{lemma}
\begin{proof}
Note that each entry within the tensor $\E_{X \sim D}\lp[(X - \mu)^{\otimes t}\rp]$ is 
of the form $\E_{X \sim D} \lp[ T( x - \mu ) \rp]$, where $T:\R^d \mapsto \R$ is some monomial of degree $t$.
Fix some degree $t$ monomial $T:\R^d \mapsto \R$, and consider the random variable
$Y = T( X - \mu )$, where $X \sim D$. 
The corresponding entry within the tensor $\E_{i \sim [m]}\lp[(X_i - \mu)^{\otimes t}\rp]$
has the same distribution as the average of $m$ \iid copies of $Y$.
We will bound from above the variance of $Y$.
We will need the following claim regarding expectations of monomials.
\begin{claim}
\label{clm:monomial-to-central}
Let $t \in Z_+$ be an even integer.
Suppose the distribution $D$ has its $t$-th central moments bounded from above by $M$.
Let $T: \R^d \mapsto \R$ be a monomial of degree $t$.
Then it holds
$$
\E_{ X \sim D } \lp[ T(X - \mu) \rp]
\leq t M.
$$
\end{claim}
\begin{proof}
Suppose $T(X - \mu) = \prod_{i=1}^d (X_i - \mu_i)^{s_i}$, where $\sum_{i=1}^d s_i = t$.
Then  we have
\begin{align*}
\E_{X \sim D} \lp[ T(X - \mu) \rp]
&\leq 
\E_{X \sim D} \lp[ \lp( \max_{i \in [d]: s_i > 0} \lp| X_i - \mu_i \rp|  \rp)^t \rp]\\
&= 
\E_{X \sim D} \lp[ \max_{i \in [d]: s_i > 0} \lp(  \lp| X_i - \mu_i \rp|^t\rp) \rp]\\
&\leq \sum_{i\in[d]: s_i > 0 } 
\E_{X \sim D} \lp[   \lp( X_i - \mu_i \rp)^{t} \rp]
\leq  t M \, ,
\end{align*}
where in the first inequality we bound 
$(X_i - \mu_i)$ from above by $\max_{i \in [d]: s_i > 0} \lp| X_i - \mu_i \rp|$,
 in the second inequality we bound the maximum of a set of non-negative numbers by their sum,
 and in the last inequality we use the fact that there are at most $t$ non-zero $s_i$'s, and that 
$D$ has its $t$-th central moments bounded from above by $M$.
This concludes the proof of \Cref{clm:monomial-to-central}.
\end{proof}
We can therefore bound from above the variance of $Y$ by
\begin{align*}
\Var[Y] \leq \E[Y^2]
= \E[ T^2(X - \mu) ]
\leq O \lp( t  F \rp) \, ,
\end{align*}
where in the last inequality we note that $T^2$ is a degree $2t$ monomial, and thus we can apply \Cref{clm:monomial-to-central}.
Hence, by Chebyshev's inequality, we have that
\begin{align*}
\abs{\E_{i \sim [m]}\lp[ T(X_i - \mu)\rp]
- 
\E_{X \sim D}\lp[ T(X - \mu)\rp]}
\leq O \lp( d^{t} \sqrt{t F / m } \rp) \, ,
\end{align*}
with probability at least $1 - o\lp(d^{-t}\rp)$.
It then follows from the union bound that
\begin{align}
\label{eq:moment-concentration}
\left \| \E_{i \sim [m]}[(X_i - \overline{\mu})^{\otimes t}] -  \E_{X \sim D}[(X - \mu)^{\otimes t}] 
\right \|_{\infty} \leq O \lp( d^{t} \sqrt{t F / m } \rp).
\end{align}
with high constant probability.
Lastly, we bound from above $\normt{ \mu - \bar \mu }$.
Since $D$ has its covariance bouned from above by $\kappa I$, it holds that the random vector $\mu - \bar \mu$ has mean $0$ and covariance bounded from above by $\kappa/m I$. Hence, the expected squared $\ell_2$ norm of the vector is at most $ \kappa d / m $. It then follows from Markov's inequality that $    \| \mu - \overline{\mu} \|_2
    \leq  O \lp( \sqrt{ \kappa d / m} \rp)$  holds with high constant probability.
This concludes the proof of \Cref{lem:basic_linf_consc-full}.
\end{proof}
\hide{
To show Lemma~\ref{lem:sos-moment-bound} we will need the following lemmas. The first is just concentration of tensors for moment bounded distributions (as shown in Lemma 3.5 of~\cite{DKKPP22}). 
\begin{lemma}
[Lemma 3.5 of ~\cite{DKKPP22}]
\label{lem:basic_linf_consc-full}
	Let $D$ be a distribution over $\R^d$ with mean $\mu$. 
 Suppose that for all $s \in [1,\infty)$, $D$ has its $s^{th}$ moment bounded by $(f(s))^s$ for some non-decreasing function $f:[1,\infty) \to \R_+$, in the direction $e_j$, i.e., suppose that for all $j \in [d]$ and $X \sim D$:
$ \E_{ X \sim D } \lp[ \lp|(X-\mu)_j\rp|^s  \rp]
 \leq \lp( f(s) \rp)^s$.
	Let $X_1, \dots, X_m$ be $m$ i.i.d.\ samples from $D$ and define $\overline{\mu}:= \sum_{i=1}^m X_i$. 
 \lnote{If we don't have sub-exponential assumption, I don't think we should have logs.}
 The following are true:
	\begin{enumerate}
	\item If $m \geq \max\left(\frac{1}{\delta^2}, 1\right) C  \left( t \log (d/\gamma)\right)\left(2f(t^2\log(d/\gamma))\right)^{2t} \max\left( 1, \frac{1}{f(t)^{2t}}  \right)$, then with probability $1 -\gamma$, we have that 
\begin{align*}
\left \| \E_{i \sim [m]}[(X_i - \overline{\mu})^{\otimes t}] -  \E_{X \sim D}[(X - \mu)^{\otimes t}] \right \|_{\infty} \leq \delta \;.
\end{align*}

	\item {If
		$m  > C (k/ \delta^2)  \log(d/\gamma) ( f(\log(d/\gamma))  )^2 $,
		 $\left\| \overline{\mu}- \mu \right\|_{2,k} \leq \delta$ with probability at least $1 - \gamma$.
}
		\end{enumerate}
\end{lemma}
}
The next lemma provides a sum of square proof that bounds from above the square of a polynomial in terms of its coefficients. 
\begin{lemma}\label{claim:upper-bound}
Let $p(v) = v^{\otimes t} A v^{\otimes t}  $ for some $d^t \times d^t$ matrix $A$ with $\|A\|_{\infty} \leq a$.  
Then
\[ \sststile{2t}{v} p(v) \leq   
a d^t \|v\|_2^{2t}. 
\]
\end{lemma}
\begin{proof} 
Since the Frobenious norm of $A$ is at most $a d^t$, we have that
$A$ is bounded from above by $a d^t I$ in Lowner order.
Thus, we can write $v^{\otimes t} ( a d^t I) v^{\otimes t} - v^{\otimes t} A v^{\otimes A}$ as a sum of squares by diagonalizing $I$ and $A$.
The lemma then follows by noting that the expression is exactly $a d^t \|v\|_2^{2t} - p(v)  $.
\end{proof}
We can now put these together to get the lemma we need. 

\begin{lemma}
\label{lem:pop-to-empirical}
Let $D$ be a distribution over $\mathbb{R}^d$ with mean $\mu$ and $t$ be a positive even integer.
Assume that (i) the covariance of $D$ is bounded from above by 
$ \kappa I $, (ii) the degree-$2t$ central moments of $D$ is bounded from above by $F > 0$, and (iii) there exists $M > 0$ such that  $D$ has $(M, t, K)$-certifably bounded moments.
Let $S= \{X_1, \ldots, X_m\}$ be a set of $m$ i.i.d.\ samples from $D$, $D'$ be the uniform distribution over $S$, and $\overline{\mu}:=\E_{X \sim D'}[X]$.
If $m \gg (t d)^{4 t} ( F / M^2 )
+  d \kappa M^{-2/t}$, then $D'$ will have $(2^{t+2} M, t, K)$-certifiably bounded moments with probability at least $0.9$.
 \end{lemma}
 \begin{proof}
 From \Cref{lem:basic_linf_consc-full} and that $m \gg (t d)^{4 t} ( F / M^2 )
+  d \kappa M^{-2/t}$, we have that the $\ell_{\infty}$ norm of the difference between the expected and empirical $t$-th tensors 
$(X - \mu)^{\otimes t}$
of $D$ and $D'$ is small, i.e.,
\begin{align}
&\left 
 \| \E_{i \sim [m]}[(X_i - \mu)^{\otimes t}] -  \E_{X \sim D}[(X - \mu)^{\otimes t}] \right \|_{\infty}\leq \frac{M}{\sqrt{d^t}} \, ,
 \label{eq:tensor-bound}
\end{align}
and that the empirical mean and the distribution mean are close, i.e., 
\begin{align}
 &\| \mu - \overline \mu\|_2 \leq M^{1/t} 
 \label{eq:mu-bound}    
\end{align}
with high constant probability.

Let $q(v) := 
\E_{i \sim [m]} [ \langle v , X_i - \mu \rangle^t ]
- 
\E_{X \sim D} [ \langle v , X - \mu \rangle^t ]$.
Combining \Cref{claim:upper-bound} and \Cref{eq:tensor-bound} gives that 
	\begin{align}
    \sststile{t}{v}& \E_{i \sim [m]} \Brac{\iprod{ v,  X_i - {\mu}}^{t}} - \E_{X \sim D}\Brac{\iprod{v, X - \mu}^{t}}
    \notag \\
		&\leq \sqrt{d^{t}} \|v\|_2^{t} \left \| \E_{i \sim [m]}[(X_i -  \mu)^{\otimes t}] -  \E_{X \sim D}[(X - \mu)^{\otimes t}] \right \|_{\infty}
  \leq \|v\|_2^{t} M.
  \label{eqn:moment_bound_1}
\end{align}
Observe that 
\begin{align}
\sststile{t}{v}
\E_{i \sim [m]} \Brac{\iprod{ v,  X_i - {\mu}}^{t}}
		&= \E_{i \sim [m]} \Brac{\iprod{ v,  X_i - {\mu}}^{t}} - \E_{X \sim D}\Brac{\iprod{v, X - \mu}^{t}}  + \E_{X \sim D}\Brac{\iprod{v, X - \mu}^{t}}  \nonumber \\
		&\leq 2 \normt{v}^t M \;,
  \label{eq:miss-center}
	\end{align}
where in the second line we use \Cref{eqn:moment_bound_1} and our assumption that $D$ has certifiably bounded central moments. 

Lastly, to prove bounded central moments of $D'$ (the uniform distribution over the samples in $S$), we note that
\begin{align*}
\sststile{t}{v}
\E_{i \sim [m]} \Brac{\iprod{ v,  X_i - \overline{\mu}}^{t}}
&\leq
2^t \E_{i \sim [m]} \Brac{\iprod{ v,  X_i - {\mu}}^{t}}
+ 
2^t \E_{i \sim [m]} \Brac{\iprod{ v,  {\mu} - \overline{\mu}}^{t}} \\
&\leq
2^{t+1} \E_{i \sim [m]} \Brac{\iprod{ v,  X_i - \overline{\mu}}^{t}}
+ 
2^t \normt{v}^{t} \normt{\mu - \overline \mu }^{t}  \\
&\leq
2^{t+2} \normt{v}^{t} M \, ,
\end{align*}
where in the first line we use the SoS triangle inequality (\Cref{fact:sos-triangle}), in the second line we use SoS Cauchy's inequality (\Cref{fact:sos-holder}), and the last inequality follows from \Cref{eq:mu-bound,eq:miss-center}.
 \end{proof}

\SOSMOMENT*
\begin{proof}
We first prove that the population version of the above inequality has SoS proof. Specifically, we show that
\begin{align}
\label{eq:population-sos}
\constrain \sststile{4k}{v}
\E_{ B \sim D_{\beta^*} } \lp[  \lp( v^{\top} \lp(Z_B - 
\E_{B' \sim T} [ Z_{B'} ]
\rp) \rp)^{2k} \rp] 
\leq  \frac{(2k)^{2k}}{n^{k}} Q \lp( \sigma^{2k} + \normt{\beta^*}^{2k} \rp).
\end{align}
We can rewrite the left hand side as
\begin{align}
& \E_{ (X_i, y_i) \sim P_{\beta^*} \forall i \in [n] } \lp[  \lp(  \lp( \frac{1}{n} \sum_{i=1}^n v^{\top} X_i y_i   - \E_{ X, y \sim P_{\beta^*} } [ v^{\top} X y ]
\rp) \rp)^{2k} \rp] \nonumber \\
&= 
\frac{1}{n^{2k}} \; 
\E_{ (X_i, y_i) \sim P_{\beta^*} \forall i \in [n] } 
\lp[    \lp(  \sum_{i=1}^n \left( v^{\top} X_i y_i   -  \; v^{\top} \beta^* \right)
\rp)^{2k} \rp].
\label{eq:factor-out-n}
\end{align}
We first show that $\E\lp[ \lp( v^\top (X y - \beta^*) \rp)^{2k} \rp]$ is SoS-certifiably bounded.
In particular, we claim that
\begin{align}
\label{eq:one-variable-bound}
\constrain \sststile{4k}{v}
\E_{ (X,y) \sim P_{\beta^*} }    
\lp[ 
\lp( v^\top (X y - \beta^*) \rp)^{2k}
\rp]
\leq (2k)^{k}  \; Q \; \lp( \sigma^{2k} + \normt{\beta^*}^{2k} \rp).
\end{align}
Note that
\begin{align*}
\constrain \sststile{4k}{v}&
\E_{ (X,y) \sim P_{\beta^*} }    
\lp[ 
\lp( v^\top (X y - \beta^*) \rp)^{2k}
\rp]
=
\E_{ (X,y) \sim P_{\beta^*} }    
\lp[ 
\lp( v^\top X X^\top \beta^* + v^\top X \xi - v^\top \beta^* \rp)^{2k}
\rp] \\
\\ 
&\leq
3^{2k}
\E_{ (X,y) \sim P_{\beta^*} }    
\lp[ 
 \lp( v^\top X X^\top \beta^* \rp)^{2k}
 + 
 \lp( v^\top X \xi \rp)^{2k}
 + \lp( v^\top \beta^* \rp)^{2k}
\rp] \, ,
\end{align*}
where in the last line we apply the SoS triangle inequality (\Cref{fact:sos-triangle}).
We then tackle the three terms separately.
For the first term, we note that
\begin{align*}
\constrain \sststile{4k}{v}
\E_{ (X,y) \sim P_{\beta^*} }    
\lp[   \lp( v^\top X X^\top \beta^* \rp)^{2k} \rp]
& \leq 
\frac{ \normt{\beta^*}^{2k}  }{2}
\E_{ (X,y) \sim P_{\beta^*} } 
\lp[  (v^\top X )^{4k} + (X^\top \beta^* / \normt{\beta^*} )^{4k}  \rp] \\
& \leq 
\normt{\beta^*}^{2k} Q.
\end{align*}
where in the first inequality we use the SoS AM-GM inequality (\Cref{fact:sos-am-gm}), and in the second inequality we use the assumption that the degree-$4k$ moments of $X$ are SoS certifiably bounded by $Q$ (\Cref{ass:dist}).
For the second term, note that
\begin{align*}
\constrain \sststile{2k}{v}
\E_{ (X,y) \sim P_{\beta^*} }    
\lp[   \lp( v^\top X \xi \rp)^{2k} \rp]
= 
\E_{ (X,y) \sim P_{\beta^*} }    
\lp[   \lp( v^\top X \rp)^{2k} \rp] \; \E[ \xi^{2k} ]
\leq (2k)^{k} \sigma^{2k} Q \, ,
\end{align*}
where in the first equality we use that $X$ and $\xi$ are independent, and in the second inequality we use again the assumption on the moments of $X$ and that the degree $2k$ moments of $\xi$ is bounded by $(2k)^{k} \sigma^{2k}$.
For the last term, we note that
$
( v^\top \beta^* )^{2k}
\leq \normt{v}^{2k} \; \normt{\beta^*}^{2k}
$ by an application of the SoS Cauchy's inequality (\Cref{fact:sos-holder}).
Combining the above analysis then shows \Cref{eq:one-variable-bound}.

By \Cref{cor:sos-iid-moment}, we then have the SoS proof 
\begin{align}
\constrain \sststile{2k}{v}
\E_{ (X_i, y_i) \sim P_{\beta^*} \forall i \in [n] } 
\lp[   \lp(  \sum_{i=1}^n \left( v^{\top} X_i y_i   -  \; v^{\top} \beta^* \right)
\rp)^{2k} \rp]
\leq n^{k} (2k)^{2k} Q \lp( \sigma^{2k} + \normt{\beta^*}^{2k} \rp)
\end{align}
Combining this with \Cref{eq:factor-out-n} then yields an SoS proof for \Cref{eq:population-sos}.

In order to establish an SoS proof for the empirical moments, we will additionally need to bound the covariance of the empirical distribution over $\{ Z_B \}_{B \in T}$.
Since  an SoS proof on the bound of the covariance is not needed, we can readily apply the $L_2-L_4$ hypercontractivity of $X$.
In particular, this shows that
$
\E[ (u^T X)^4   ] \leq O( 1 ) \lp( \E \lp[ (u^T X)^2 \rp] \rp)^2 \leq O(1)
$
for any unit vector $u$.
With an argument almost identical to the SoS bound on the degree-$2k$ moments, we can show that
$$
\E_{ (X_i, y_i) \sim P_{\beta^*} \forall i \in [n] } 
\lp[   \lp(  
\frac{1}{n} \sum_{i=1}^n \left( v^{\top} X_i y_i   -  \; v^{\top} \beta^* \right)
\rp)^{2} \rp]
\leq  O \lp( \frac{\sigma^{2} + \normt{\beta^*}^{2}}{n} \rp).
$$
This shows property (b) in the lemma.

Let $C$ be a sufficiently large constant.
The SoS proof for the empirical moments then follows
by an application of Lemma~\ref{lem:pop-to-empirical} with 
$t = 2k$, 
$\kappa := C \frac{1}{n}  \lp(\sigma^2 + \normt{\beta^*}^2 \rp)$, 
$ M := \frac{(2k)^{2k}}{n^k} \; Q \; \lp(\sigma^{2k} + \normt{\beta^*}^{2k} \rp)$,
 $F := 
\frac{(4k)^{4k}}{n^{2k}} \; Q \; \lp(\sigma^{4k} + \normt{\beta^*}^{4k} \rp)
\leq 2^{4k} M^2 Q^{-1}
$, and 
\begin{align*}
m \gg (2kd)^{8k} \; 2^{4k}  \; Q^{-1}
+ \frac{d}{n}  \lp(\sigma^2 + \normt{\beta^*}^2 \rp) \;  M^{-1/k} + 1.
\end{align*}
It is not hard to see that
$$
(2kd)^{8k} \; 2^{4k}  \; Q^{-1}
+ \frac{d}{n}  \lp(\sigma^2 + \normt{\beta^*}^2 \rp) \;  M^{-1/k} + 1
\leq 
O(1) \; 
\lp( (4kd)^{8k} Q^{-1} + d Q^{-1/k} + 1 \rp)
\leq
O \; 
\lp( (4kd)^{8k} Q^{-1} \rp) \, ,
$$
where the last inequality can be shown by examining the cases where $d Q^{-1/k} \geq 1$ and $d Q^{-1/k} < 1$ separately.
This concludes the proof of \Cref{lem:sos-moment-bound}.
\end{proof}

\section{Pruning Procedure and its Analysis}
\label{app:pruning}

The main theorem for this subsection is the following: 
\PRUNE*

The \textsf{Pruning} algorithm involves two phases: 
initially, it filters regressors $\beta \in L$ by retaining those matching a certain set of solvable linear inequalities. 
Then, it selects a subset of the remaining regressors, ensuring each pair is adequately distant. 
Lemmas \ref{lem:pruning-list-size-bound} and \ref{lem:pruning-error-bound} respectively prove that the refined list is not excessively large and contains a regressor near the optimal $\beta^*$,
given one exists in the original list $L$.
The proof of Proposition~\ref{prop:prune} follows from the above two lemmas.

For each regressor, we restate the set of linear inequalities $\IE(\beta; L, T, R)$
in the weighting function $\mathcal W$ over the set of batches $T$. 

\begin{align}
&\sum\nolimits_{B \in T}  \mathcal W(B) \geq  0.9 \alpha |T|, \label{Eq2Cond0-app}
\\
     &\forall \beta' \in L \text{ such that }
     \norm{\beta' - \beta} \geq c \lp( \betagap \rp) \text{ for some sufficiently large constant } c
     \, , \nonumber \\
     &\sum_{B \in T}
     \mathbbm 1 \lp\{ 
     \sum_{ (X,y)\in B } \lp( y-X^\top \beta\rp)^2 
     \leq 
     \sum_{ (X,y)\in B } \lp( y-X^\top \beta'\rp)^2 
     \rp \}
     \mathcal W(B)
     \leq \frac{\alpha}{20} \sum_{B \in T}  \mathcal W(B). \label{Eq2Cond1-app}
\end{align}

We now show there cannot be too many regressors whose associated linear inequalities 
are satisfiable subject to the constraint that they are all sufficiently separated.
This mainly comes from the observation that Condition~\ref{Eq2Cond1-app} 
enforces the soft clusters associated with two sufficiently separated candidate regressors must have small intersection.
\LISTSIZE*
\begin{proof}
Let $I$ be a set of weighting functions $\mathcal W: T \mapsto [0, 1]$ over batches.
We first define the union and disjoint operators for weighting functions as follows
\begin{align*}
    \lp(\bigcup_{\mathcal W \in I} \mathcal W \rp)(B) = \max_{\mathcal W \in I} \mathcal W(B) \, , \,
    \lp(\bigcap_{\mathcal W \in I} \mathcal W \rp)(B) = \min_{\mathcal W \in I} \mathcal W(B).  
\end{align*}
Moreover, for a weighting function $\mathcal W: T \mapsto [0, 1]$, we define $\mathcal W(T) = \sum_{B \in T} \mathcal W(B)$.
Let $\beta_1, \beta_2$ be two vectors from the sublist $L'$, and
$\mathcal W_1, \mathcal W_2$ be the solutions of $\IE(\beta_1; L, T, R)$ and $\IE(\beta_2; L, T, R)$ respectively.
We proceed to argue that 
$( \mathcal W_1 \cap \mathcal W_2 ) (T) < 0.1 \alpha \lp( \mathcal W_1(T) + \mathcal W_2(T) \rp)$.
For the sake of contradiction, we assume that
\begin{align}
\label{eq:set-intersection-lower-bound-app}
( \mathcal W_1 \cap \mathcal W_2 ) (T) > 0.1 \alpha \lp( \mathcal W_1(T) + \mathcal W_2(T) \rp).
\end{align}
Define the following two subsets of batches: 
\begin{align}
& \cE_1 := \left\{ B \in T:       
     \sum\nolimits_{ (X,y)\in B } \lp( y-X^\top \beta_1\rp)^2 
     \leq 
     \sum\nolimits_{ (X,y)\in B } \lp( y-X^\top \beta_2\rp)^2  \right\} \, , \nonumber \\
\text{and} \; 
& \cE_2 := \left\{ B \in T:       
     \sum\nolimits_{ (X,y)\in B } \lp( y-X^\top \beta_2\rp)^2 
     \leq 
     \sum\nolimits_{ (X,y)\in B } \lp( y-X^\top \beta_1\rp)^2 
\right\}.\nonumber
\end{align}
Since each batch $B$ belongs to either $\cE_1$ or $\cE_2$, 
we have either $\lp(\mathcal W_1 \cap \mathcal W_2\rp)( \cE_1 )
\geq \lp(\mathcal W_1 \cap \mathcal W_2\rp)( T ) / 2
$
or $\lp(\mathcal W_1 \cap \mathcal W_2\rp)( \cE_2 )
\geq \lp(\mathcal W_1 \cap \mathcal W_2\rp)( T ) / 2
$.
Without loss of generality, assume that we are in the former case.
This then implies that
\begin{align*}
&\sum_{B \in T}
     \mathbbm 1 \lp\{ 
     \sum\nolimits_{ (X,y)\in B } \lp( y-X^\top \beta_1\rp)^2 
     \leq 
     \sum\nolimits_{ (X,y)\in B } \lp( y-X^\top \beta_2\rp)^2 
     \rp \}
     \mathcal W_1(B) \\
     &\geq 0.05 \alpha \lp(  \mathcal W_1(T) + \mathcal W_2(T) \rp)
     > \frac{\alpha}{20} \sum_{B \in T} \mathcal W_1(T) \, ,    
\end{align*}
which contradicts \Cref{Eq2Cond1-app} for $\beta_1$.
This shows the opposite of \Cref{eq:set-intersection-lower-bound-app}.

    Lastly, assume that there are more than $4/\alpha$ many candidate regressors in the sublist $L'$ for the sake of contradiction.
    Arbitrarily pick $\ell = \ceil{4/\alpha}$ many regressors from $L'$, and let $\mathcal W_1,\ldots,\mathcal W_{\ell}$ be the solutions to the linear inequalities associated with the candidate regressors picked. 
    Then,
    \begin{align*}
    |T| &\geq 
        \lp( \bigcup_{i=1}^{\ell}
        \mathcal W_i \rp)(T) \\
        &\geq \sum_{i=1}^\ell \mathcal W_i(T)  - \sum_{i < j \in [\ell]}  (\mathcal W_i \cap \mathcal W_j)(T)  \\
        &\geq \sum_{i=1}^{\ell} \mathcal W_i(T) - 0.1 \alpha \sum_{i < j \in [\ell]}(\mathcal W_i(T) + \mathcal W_j(T)) \\
        &= \left( 1 - 0.1 \alpha (\ell-1)  \right)\sum_{i=1}^\ell \mathcal W_i(T)  \\
        &\geq \left( 1 - 0.1(\ell-1)\alpha \right) \ell (0.9\alpha) |T| \\
        &\geq 2.88 |T| \;,
    \end{align*}
    where in the first line we use the fact that the weights are bounded from above by $1$, 
    in the second line we use the approximate inclusion-exclusion principle, in the third line we use the opposite of \Cref{eq:set-intersection-lower-bound-app}, in the fourth line we use the elementary fact that $\sum_{i \neq j \in [\ell]} (x_i + x_j) = (\ell-1) \sum_{i=1}^\ell x_i$, in the fifth line we use $\mathcal W_i(T) \geq 0.9 \alpha |T|$ as they need to satisfy Condition~\ref{Eq2Cond0}, and in the last line we use the definition of $\ell = \ceil{4/\alpha}$.
    This is clearly a contradiction, and hence concludes the proof of \Cref{lem:pruning-list-size-bound}.
\end{proof}

Next we show that the set of linear inequalities constructed for some $\beta$ admit solutions with high probability as long as $\beta$ is close to $\beta^*$.
\ERROR*

To prove \Cref{lem:pruning-error-bound}, we will make essential use of the following anti-concentration inequalities.

\begin{fact}[Paley–Zygmund Inequality]
 If $Z \geq 0$ is a positive random variable with finite variance, and $\theta \in [0, 1]$, then it holds
 $$
 \Pr[Z \geq \theta \E[Z]]
 \geq (1 - \theta)^2 \frac{ \E[Z]^2 }{ \E[Z^2] }.
 $$
\end{fact}
Combining the above with our distributional assumption that the clean covariates distribution satisfies $L_2$-$L_4$ hypercontractivity, we obtain the following \emph{weak anti-concentration} property.
\begin{corollary}[Weak Anti-concentration]
\label{cor:weak-anti}
Let $v$ be a unit vector in $\R^d$, and $X$ be a random unit vector satisfying \Cref{ass:dist}. Then it holds
$$
\Pr[ (v X)^2  \geq 1/2 ] \geq \Omega(1).
$$
\end{corollary}

We are now ready to give the proof of \Cref{lem:pruning-error-bound}.

\begin{proof}[Proof of \Cref{lem:pruning-error-bound}]
Let $\beta$ be a regressor within the list such that
$ \normt{\beta - \beta^*} < R $.
Our goal is to show that the associated linear inequalities 
$IE(\beta;L,T,R)$ admits solutions.
In particular, we claim that setting $\mathcal W(B) = 1$ for all inlier batch $B$ and $\mathcal W(B) = 0$ for all outlier batch $B$ gives a solution. 
Condition~\ref{Eq2Cond0} is satisfied since in expectation there should be $\alpha$-fraction of inlier batches. Since we take $C \log(\delta/\alpha) / \alpha^2$ many batches, the actual fraction of inlier batches should be at least $0.9 \alpha$ with probability at least $1 - \delta$ when $C$ is sufficiently large by the Chernoff bound.

Next we show Condition~\ref{Eq2Cond1} is satisfied with high probability over the randomness of $T$.
Fix some $\beta'$ satisfying 
$\norm{\beta' - \beta}_2 \gg \betagap$.
We will analyze the random variable 
$$
Z_{\beta'}(B) :=
\sum_{ (X,y) \sim B } 
\lp( y - X^{\top} \beta' \rp)^2
 - 
 \sum_{ (X,y) \sim B } 
\lp( y - X^{\top} \beta \rp)^2,
$$
where $B \sim D_{\beta^*}$.
Recall that we have $y = X^{\top} \beta^* + \xi$, where $\xi \sim \normal(0, \sigma^2)$.
We will rewrite $Z(\beta')$ slightly with the random variables 
$\{(X^{(i)}, \xi^{(i)})\}_{i=1}^n$, where each $X^{(i)}$ is drawn independently from a distribution satisfying \Cref{ass:dist} 
, and each $\xi^{(i)}$ is independently distributed as $\normal(0, \sigma^2)$.
We thus have that
$$
Z_{\beta'}(B)
=
\sum_{i=1}^n
\lp( (\beta' - \beta^*)^{\top} X^{{(i)}}  \rp)^2
- 
\lp( (\beta - \beta^*)^{\top} X^{{(i)}} \rp)^2
+ 
2 \xi^{(i)} \; \lp( \beta - \beta' \rp)^{\top} X^{(i)}.
$$

Denote the three terms in the summation by:
\begin{align*}
Z_1:= \sum_{i=1}^n
\lp( (\beta' - \beta^*)^{\top} X^{{(i)}}  \rp)^2 \, ,
Z_2:= \sum_{i=1}^n \lp( (\beta - \beta^*)^{\top} X^{{(i)}} \rp)^2 \, ,
Z_3:= \sum_{i=1}^n 2 \xi^{(i)} \; \lp( \beta - \beta' \rp)^{\top} X^{(i)}.
\end{align*}
We proceed to argue that $Z_1$ is bounded from below, and $Z_2$, $Z_3$ are bounded from above with high probability.
\footnote{Note that there are correlations between $Z_1,Z_2,Z_3$. Nonetheless, these correlations will not affect our analysis.}.  

For $Z_1$, applying the weak anti-concentraiton property of $X$
(\Cref{cor:weak-anti}) gives that
$$
\Pr\lp[  \left( (\beta' - \beta^*)^\top X^{(i)} \right)^2 
\geq \normt{ \beta' - \beta^* }^2/2
\rp] \geq  \gamma.
$$
for some universal constant $\gamma$.
By the Chernoff bound, given that $n \gg \log(1 / \alpha)$, the fraction of $X^{(i)}$ such that 
$\left( (\beta' - \beta^*)^\top X^{(i)} \right)^2 
\geq \normt{ \beta' - \beta^* }^2/2$ will be at least $\gamma / 2$ with probability at least $1 - \alpha/120$.
It then follows that
\begin{align}
\label{eq:new-A-concentration-app}
\Pr \lp[ Z_1 \leq \frac{ \gamma n }{4} \normt{ \beta' - \beta^* }^2 \rp] \geq 1 - \alpha / 120.
\end{align}
For $Z_2$, since $\E [ X^{(i)} \lp( X^{(i)} \rp)^\dagger ] = I $ by \Cref{ass:dist}, it follows that
$\E[Z_2] = n \normt{ \beta - \beta^* }^2$.
In order to show that $Z_2$ is sufficiently concentrated, we will bound from above the $k$-th central moments of $Z_2$ for some even integer $k \leq \Delta$.
Define $y_i = \lp( ( \beta - \beta^* )^\top X^{(i)} \rp)^2$. 
We note that the $y_i$s are \iid random variables with their degree-$k$ central moments bounded from above by
\begin{align*}
\E \lp[  
\lp( 
\lp( ( \beta - \beta^* )^\top X^{(i)} \rp)^2
- \normt{ \beta - \beta^* }^2
\rp)^{k}
\rp]
&\leq 
2^{k}
\E \lp[  
\lp( ( \beta - \beta^* )^\top X^{(i)} \rp)^{2k}
+ \normt{ \beta - \beta^* }^{2k}
\rp] \\
&\leq 
2^{k+1} \; \normt{ \beta - \beta^* }^{2k} \; Q
\end{align*}
where in the first line we apply the triangle inequality (\Cref{fact:sos-triangle}), 
and in the second line we use the assumption on the degree $2k$ moments of $X$.
Hence, applying \Cref{cor:sos-iid-moment} gives that 
the degree $k$ moment of $Z_2$ is bounded from above by
\begin{align*}
    \E\lp[  \lp(  \sum_{i=1}^n \lp( y_i - \E[y_i] \rp) \rp)^k  \rp]
    \leq  2 \; (4kn)^{k/2}  \; \normt{ \beta - \beta^* }^{2k} \; Q  \;.
\end{align*}
In other words, we have that
\begin{align*}
\left( \E\lp[  \lp(  \sum_{i=1}^n \lp( y_i - \E[y_i] \rp) \rp)^k  \rp] \right)^{1/k}
    \leq  2^{1/k} \sqrt{4kn} \; Q \normt{ \beta - \beta^* }^{2}.
\end{align*}
By Chebyshev's inequality, we thus have that
\begin{align}
\label{eq:new-B-concentration-app}
\Pr\lp[ Z_2 \geq   \normt{ \beta - \beta^* }^{2}
\lp( n +  100 \sqrt{kn} \; Q^{1/k}  \rp)
\rp]  \leq \alpha / 120.
\end{align}

For $Z_3$, note that $\E[Z_3] = 0$ since $\xi^{(i)}$ has mean $0$.
To argue for its concentration, we again proceed to bound its degree-$k$ moment for some even integer $k$.
Similarly, we define $z_i = \xi^{(i)} ( \beta - \beta')^\top X^{(i)}$.
The degree-$k$ central moments of $z_i$ can be bounded from above by
\begin{align*}
\E \lp[  
\lp(
\xi^{(i)} ( \beta - \beta')^\top X^{(i)}
\rp)^{k}
\rp]
&= 
\E \lp[  
  \lp( \xi^{(i)} \rp)^{k}
\rp]
\E \lp[ 
\lp(  ( \beta - \beta')^\top X^{(i)} \rp)^k
\rp].
\end{align*}
We can apply the upper bounds on the degree-$k$ moments of $\xi^{(i)}$ and  $X^{(i)}$ respectively.
This allows us to conclude that
$$
\E \lp[  
\lp(
\xi^{(i)} ( \beta - \beta')^\top X^{(i)}
\rp)^{k}
\rp] 
\leq k^{k/2} Q \sigma^k \normt{ \beta - \beta' }^k.
$$
Applying \Cref{cor:sos-iid-moment} then gives that
\begin{align*}
\E \lp[ 
\lp( \sum_{i=1}^n z_i \rp)^k
\rp]
\leq n^{k/2} \; k^k  Q \sigma^k \normt{ \beta - \beta' }^k.
\end{align*}
In other words, we have that
$$
\lp( 
\E \lp[ 
\lp( \sum_{i=1}^n z_i \rp)^k
\rp] \rp)^{1/k}
\leq k \sqrt{n} \; \sigma  \; \normt{ \beta - \beta' } Q.
$$
By Chebyshev's inequality, it holds that
\begin{align}
\label{eq:new-C-concentration-app} 
\Pr 
\lp[
Z_2 > 10 k \sqrt{n} \; \sigma Q \; \normt{ \beta - \beta' }
\rp] \leq \alpha/120.   
\end{align}

By the union bound, the events in \Cref{eq:new-A-concentration-app}, \Cref{eq:new-B-concentration-app}, and \Cref{eq:new-C-concentration-app} are satisfied simultaneously with probability at least $1 - \alpha/40$. When that happens, $Z_{\beta'}(B)$ will be bounded from below by
\begin{align}
\label{eq:concentration-lb}
\frac{ \gamma  }{4} \; n \; \normt{\beta ' - \beta^*}^2
- \normt{\beta - \beta^*}^2 \; \alpha^{-1/k} \;
\lp( n +  100 \sqrt{kn} \; Q^{1/k}  \rp)
- 10 k \sqrt{n} \; \sigma Q \; \alpha^{-1/k} \; \normt{ \beta - \beta' }.
\end{align}
First, we claim that 
\begin{align}
\label{eq:first-distance}
&\normt{\beta' - \beta^*} 
\gg \normt{\beta - \beta^*} \\
\label{eq:second-distance}
&\normt{\beta' - \beta^*} 
\geq (1 - o(1))  \normt{\beta - \beta'}.
\end{align}
To prove \Cref{eq:second-distance},
we note that
$$\normt{\beta' - \beta^*}  
\geq \normt{\beta' - \beta} - \normt{\beta - \beta^*}
\geq (1 - o(1))  \normt{\beta - \beta'}.
$$ 
where the first inequality is the triangle inequality, and the second inequality is true by our assumption that $\normt{\beta - \beta^*} < R \ll \normt{\beta' - \beta}$.
\Cref{eq:first-distance} then follows immediately as 
$  \normt{\beta - \beta'} \gg  \normt{\beta - \beta^*}$.

With the above inequalities in mind, we proceed to argue that the positive term dominates all the negative terms in \Cref{eq:concentration-lb}.
Since $\gamma$ is a universal constant,
it follows that
$$
\gamma n \; \normt{\beta ' - \beta^*}^2
\gg 
\normt{\beta - \beta^*}^2 n.
$$
Next recall that
$ n \gg \prunebs $ by our assumption on $n$.
It then follows that 
$$
\gamma n \; \normt{\beta ' - \beta^*}^2
\gg 
\normt{\beta - \beta^*}^2 \; 100 \sqrt{kn} \; Q^{1/k} \; \alpha^{-1/k}.
$$
Lastly, recall that we assume $ \normt{\beta - \beta'} \gg 
k \sigma Q \; \alpha^{-1/k} / \sqrt{n}.
$
Combining this with \Cref{eq:second-distance} and \Cref{eq:first-distance} then gives that
$ \normt{\beta' - \beta^*} \geq (1 - o(1)) \; \normt{ \beta - \beta' } \gg k \sigma Q \; \alpha^{-1/k} / \sqrt{n}$, which implies that
$ \normt{\beta' - \beta^*}^2 \gg \; \normt{ \beta - \beta' } k \sigma Q \; \alpha^{-1/k} / \sqrt{n}$.
It then follows that
$$
\gamma n \; \normt{\beta ' - \beta^*}^2
\gg 
10 k \sqrt{n} \; \sigma Q \; \alpha^{-1/k} \; \normt{ \beta - \beta' }
$$
Combining the above gives that 
\begin{align}
\label{eq:population-Z-violation-app}
\Pr_{B \sim \D_{\beta^*}} \lp[ Z_{\beta'}(B) > 0 \rp] > 1 - \alpha / 40 \, ,    
\end{align}
as long as $n \gg \prunebs$ and 
$ \normt{\beta - \beta'} \gg 
\betagap$.

    Since the inlier batches are all drawn independently, it holds
    the faction of inlier batches violating the condition
    is at most $\alpha/20$ with probability at least $1 - \delta / K^2$ when the number of inlier batches drawn are at least $N \gg \log(K /\delta )  \alpha^{-2}$.
    Since the size of $L$ is at most $K$,
    there are at most $K-1$ many $\beta'$ we need to consider.
    Condition~\ref{Eq2Cond1} is therefore satisfied with probability at least $1 - \delta$ by the union bound.

    When we have $\log(K) > d^2$, we will need an alternative argument.
    We note that $Z_1$ and $Z_2$ are both linear functions in the 
    random variables $\sum_{i=1}^n X^{(i)} {X^{(i)}}^T$ of dimension $d^2$, and $Z_3$ is a linear function in the random variables $\sum_{i=1}^n \xi^{(i)} \; X^{(i)} $.
    Thus, overall, for any $\beta' \in \R^d$, $Z_{\beta'}(B)$ is a linear function in $O\lp(d^2\rp)$ many random variables.
    It then follows that, for any $\beta' \in \R^d$, $\mathbbm 1 \{ Z_{\beta'}(B) > 0 \}$ is an $O(d^2)$-dimensional linear threshold function, which has VC-dimension $O(d^2)$.
    Let $G$ be $N' \gg d^2 \alpha^{-2} \log(1/\delta)$ many inlier batches drawn from $D_{\beta^*}$. 
    By the VC-inequality,  we thus have
    \begin{align*}
        \Pr_{G}\lp[  
        \sup_{ \beta' \in \R^d }
        \lp |
        \Pr_{ B \sim G }
        \lp[ Z_B(\beta') > 0 \rp]
        - 
        \Pr_{ B \sim D_{\beta^*} }
        \lp[ Z_B(\beta') > 0 \rp]
        \rp |
        > \alpha / 20
        \rp] \leq \delta.
    \end{align*}
    Combining this with \Cref{eq:population-Z-violation-app} then
    shows that Condition~\ref{Eq2Cond1} is satisfied with probability at least $1 - \delta$.
\end{proof}

\hide{
\section{Removing $R$ dependency from \Cref{thm:main}}
It suffices to show that we can construct an initial hypothesis list of size $\poly(1/\alpha)$ containing some candidate regressor $\hat \beta$ such that $ \norm{ \beta - \beta^* }_2 \leq \poly(d)$.
After that, we can run the pruning routine (\Cref{prop:prune}) to reduce it into a list of size $O(1/\alpha)$.
This then allows us to replace $R$ with $\poly(d)$ in the base case of the inductive argument of \Cref{thm:main}.

Let $\{ (X_i, y_i) \}_{i=1}^m$, where $m = 100 \alpha$, be $m$ \iid samples drawn from the initial linear regression instance.
It is then guaranteed that at least one of the samples drawn will be an inlier.
For an inlier sample $(X,y)$, we note that 
$ X y $ has expectation $\beta^*$, and sub-exponential tail as shown in [REF].
Thus, with high constant probability, we have 
$\normt{X y - \beta^*} \leq \poly(d)$. 
This complete the proof.}

\section{Reduction from the Batch-Setting to the Non-Batch Setting}
\label{app:reduction}

We point out a simple reduction (in \Cref{cl:reduction}), which allows one to solve list-decodable linear regression in the non-batch setting using an algorithm for the batch-setting in a black-box manner. The idea is the trivial observation we can construct our own batches of size $n$ just by collecting together $n$ individual labeled examples. 
Denote by $\alpha$ the probability that an individual example is inlier. Then the probability that a batch made in the aforementioned way consists only of inliers is $\alpha_B = \alpha^n$. Then, running any  algorithm designed for the batch setting should yield guarantees where the corruption rate is being replaced by $\alpha^n$. In particular, if we denote by $m(\alpha_B,d),\ell(\alpha_B)$ and $ \mathrm{error}(\alpha_B)$ the sample complexity, list size and error guarantee of the black-box algorithm (which are functions of the corruption rate $\alpha_B$ and maybe other parameters like the dimension $d$ which do not matter for this discussion), then the resulting algorithm for solving the problem in the non-batch setting will have its sample complexity, lits size and error rate being $m(\alpha^n,d),\ell(\alpha^n)$ and $\mathrm{error}(\alpha^n)$ respectively.

For convenience, throughout this section we will restrict our ourselves to the case $\alpha < 1/2$, which corresponds to more than half of the data being  corrupted. We are interested only in this since this is the truly ``list-decodable setting''. For this reason, we will use $n \ll \log(1/\alpha_B)$ in the claim below (because we have already mentioned that $\alpha=\alpha_B^{1/n}$, thus in order to have $\alpha<1/2$ we need $n \ll \log(1/\alpha_B)$).

\begin{claim}\label{cl:reduction}
Denote by $d$ the ambient dimension and by $\alpha \in (0,1/2)$ the corruption level for the non-batch setting. Let $c>0$ be a sufficiently small absolute constant.
    Suppose that $\cA$ is an algorithm with the guarantee that for any $\alpha_B \in (0,1/2)$  it can draw $m(\alpha_B,d)$ batches of size $n=c \log(1/\alpha_B)$ from the corrupted distribution of \Cref{def:glr} with corruption level $\alpha_B$, and output a list of size $\ell(\alpha_B)$ of vectors which contains a vector $\hat{\beta}$ with $\|\hat{\beta}-\beta^{*}\|_2 \leq \mathrm{error}(\alpha_B)$.
    Then, there exists another algorithm $\cA'$ that draws $m(\alpha^n,d)$ batches of size $1$ from the corrupted distribution of \Cref{def:glr} with rate of corruption $\alpha$, and outputs a list of size $\ell(\alpha^n)$ of vectors which contains a vector $\hat{\beta}$ with $\|\hat{\beta}-\beta\|_2 \leq \mathrm{error}(\alpha^n)$.
\end{claim}

This reduction, in combination with the lower bound of \cite{DiaKPPS21}, can serve as informal evidence that doing  list-decodable linear regression with batch sizes $n \ll \log(1/\alpha_B)$ likely requires exponential time. In particular, Theorem 1.5 in \cite{DiaKPPS21} provides evidence\footnote{By the term ``evidence'' we mean that the lower bound applies to the Statistical Query model. Although this does not imply hardness results for all efficient algorithms, SQ lower bounds have long served as strong indication of computational hardness.} that any algorithm with polynomial sample complexity needs exponential list-size or exponential runtime. Let $n = c\log(1/\alpha_B)$ for some constant $c\ll 1$. If \Cref{thm:main} were to allow for that batch size of  $n = c\log(1/\alpha_B)$ (recall that it right now only works for $n \gg \log(1/\alpha_B)$), 
then, by the reduction above (\Cref{cl:reduction}) we would obtain an algorithm for the non-batch setting, with quasi-polynomial runtime and list size which would contradict the hardness evidence.

\end{document}